\documentclass{article}
\usepackage{slp_preprint,times}
\usepackage[hidelinks,hypertexnames=false]{hyperref}
\usepackage{url}
\usepackage{amsmath,amssymb,amsthm,mathtools}
\usepackage{booktabs,longtable,tabularx,array,multirow,makecell}
\usepackage{microtype}
\usepackage{graphicx}
\usepackage{xcolor}
\usepackage{float}
\usepackage{tikz}
\usepackage{orcidlink}
\usetikzlibrary{positioning,arrows.meta,calc,fit}
\usepackage[T1]{fontenc}
\newcolumntype{Y}{>{\raggedright\arraybackslash}X}
\newcommand{\R}{\mathbb{R}}

\newcommand{\Cset}{\mathcal{C}}
\newcommand{\Xset}{\mathcal{X}}
\newcommand{\Zset}{\mathcal{Z}}
\newcommand{\Yset}{\mathcal{Y}}
\newcommand{\push}{\mathbin{\#}}
\newcommand{\E}{\mathbb{E}}

\newcommand{\SLPProbHard}{\textmd{\texttt{SLP-ProbHard}}}
\newcommand{\SFLP}{\textmd{\textit{SFLP}}}
\newcommand{\pospart}[1]{\left[#1\right]_+}
\newcommand{\psiplus}{\psi_{+}}
\newcommand{\psibm}{\psi_{\mathrm{bm}}}
\newcommand{\psibr}{\psi_{\mathrm{br}}}
\newcommand{\psiint}{\psi_{\mathrm{int}}}
\DeclareMathOperator{\Cov}{Cov}
\DeclareMathOperator{\Var}{Var}
\DeclareMathOperator{\softplus}{softplus}

\newtheorem{theorem}{Theorem}
\newtheorem{proposition}{Proposition}
\newtheorem{corollary}{Corollary}
\newtheorem{remark}{Remark}

\title{%
  \raggedright\normalfont\LARGE
  \texttt{SLP-ProbHard}: Probabilistic Hard-Constrained Learning via
  Structural Latent Parameterization
}

\author{%
  Wondesen Teshome Bekele\textsuperscript{1,2,*}\,
  \orcidlink{0000-0002-3827-8919}
  \qquad
  Marco D'Oria\textsuperscript{1}\,
  \orcidlink{0000-0002-5154-7052}\\[4pt]
  \normalfont\small
  \textsuperscript{1}Department of Engineering and Architecture,
  University of Parma, Parma, Italy\\
  \normalfont\small
  \textsuperscript{2}University School for Advanced Studies
  IUSS Pavia, Pavia, Italy\\[2pt]
  \normalfont\small
  \textsuperscript{*}Corresponding author:
  \href{mailto:wondesenteshome.bekele@unipr.it}
       {wondesenteshome.bekele@unipr.it}
}
\slpfinalcopy
\hypersetup{pdftitle={SLP-ProbHard: Probabilistic Hard-Constrained Learning via Structural Latent Parameterization}, pdfauthor={Wondesen Teshome Bekele; Marco D'Oria}}
\begin{document}
\maketitle
\lhead{Preprint}
\rhead{}

\begin{abstract}
Many probabilistic predictors must satisfy exact structure in every stochastic realization, yet common hard-constraint approaches form predictions in ambient coordinates and subsequently complete, correct, or project them. We introduce \SLPProbHard{}, a cross-family representation-centered framework for probabilistic hard-constrained learning when suitable explicit structural parameterizations are available. Its core modelling object is a Structural Feasible Latent Parameterization (\SFLP{}), consisting of a structural latent law \(Z\sim P_\theta^Z(\cdot\mid x)\) and an explicit feasible map \(Y=h_\phi(x,Z)\), with \(h_\phi(x,z)\in\mathcal C(x)\) for every \(z\in\mathcal Z\). Together, the structural latent law and feasible map define the predictive model itself rather than a final feasibility wrapper. The latent law and map jointly determine predictive support and boundary probability; the chosen coordinates and map specify the stochastic representation and shape attainable dependence, calibration, expressiveness, and computation. The formulation does not prescribe a particular predictive backbone or latent distribution family; the present experiments instantiate it with Gaussian latent laws and fixed geometry-derived maps. We instantiate the formulation across representative affine equalities, convex structural sets (ordering, simplex, and nonnegative ordering), explicit nonlinear equalities/manifolds (circle and scale-shape), and three structural views of the same seven-basin hydrological FDC data. A verified official-source ProbHardE2E/DPPL affine comparison gives zero practical violations for both methods: \SLPProbHard{} uses 8 rather than 11 stochastic coordinates and improves MSE/MAE, whereas DPPL attains better marginal CRPS and higher, closer-to-nominal coverage; the paired test does not detect an Energy Score difference across ten seeds. Real-world affine and nonlinear FDC representations reduce 13 to 7 and 14 ambient to 8 computational coordinates, respectively. Together, these results show that exact feasibility alone does not determine a predictive law and motivate direct structural generation when scientifically meaningful feasible coordinates are available.
\end{abstract}

\section{Introduction}
Probabilistic models are increasingly used where uncertainty matters but output structure is non-negotiable. In these settings feasibility must hold not only for a predictive mean but for every stochastic realization. Examples include conservation and aggregation identities, non-negativity, ordering, simplex constraints, explicit nonlinear equalities, composed constraints, and initial or boundary conditions.

A common strategy is to predict in an ambient space and then enforce feasibility through a soft penalty, post-processing rule, optimization layer, or differentiable projection/correction~\citep{amos2017,agrawal2019,donti2021}. Recent methods broaden this route to convex, semidefinite, nonlinear equality, and nonlinear inequality constraints~\citep{tordesillas2023,grontas2025,tang2026,goertzen2026,chu2026,wang2026diffslack}; ProbHardE2E is the closest dedicated probabilistic projection framework~\citep{utkarsh2025}. Projection is broadly useful, especially when a feasible set is implicit or difficult to parameterize.

\SLPProbHard{} asks the complementary question: \emph{when a scientifically meaningful generative representation is available, can the predictive distribution be defined directly in those feasible coordinates?} We use
\begin{equation}
Z \sim P_\theta^Z(\cdot \mid x), \qquad
Y = h_\phi(x,Z), \qquad
h_\phi(x,z) \in \mathcal{C}(x)
\quad \forall z \in \mathcal{Z}.
\end{equation}
so the output law is the pushforward $P^Y_\theta=h_{\phi,x}\push P^Z_\theta$. The structural map is therefore not merely a final feasibility wrapper: together with the latent law it specifies which feasible outputs can occur, which boundary events can receive probability, and how many stochastic coordinates are actually modeled.

\paragraph{Novelty scope.}
\SLPProbHard{} introduces a cross-family representation-centered framework in which a structural latent law and an explicit feasible map jointly define the predictive law. Its methodological contribution is to formalize and test how this representation changes feasible support and boundary probability, stochastic-coordinate structure, uncertainty propagation, attainable dependence/covariance, expressiveness, calibration, and computation. Thus, two constructions can satisfy the same hard constraint while defining different predictive-law families. The individual feasibility constructions used here, null-space coordinates, nonnegative increments, normalization, trigonometric coordinates, and the associated pushforward identities, use established mathematical primitives. Important structural-probabilistic precedents also exist in hierarchical forecasting, general linear constraints, hard linear generative learning, and domain-specific stochastic ordering ~\citep{olivares2024clover, girolimetto2024general,li2026hardlinear, bekele2026smthcsd}. The contribution of \SLPProbHard{} is not these primitives individually, but their organization as a common probabilistic model-design methodology across representative affine, convex-structural, and explicit nonlinear constraint geometries.

\paragraph{Contributions.}
We make three connected contributions.

\textbf{(1) Predictive-law abstraction.} We introduce \SLPProbHard{} and \SFLP{} as its core modelling abstraction: a structural latent predictive law and an explicit feasible map jointly define a sample-wise hard-feasible pushforward law.

\textbf{(2) Representation audit and consequences.} We characterize and empirically validate how \SFLP{} choice affects support and boundary probability, intrinsic versus computational stochastic dimension, uncertainty propagation, attainable dependence/covariance, calibration, parameterization bias, and computation across representative constraint geometries.

\textbf{(3) Cross-family empirical validation.} We evaluate the framework on affine equality, ordering, simplex, composed nonnegative ordering, explicit nonlinear manifold/scale--shape problems, and three representations of the same hydrological FDC data. A verified official-source ProbHardE2E/DPPL comparison on the compatible affine benchmark demonstrates shared practical hard feasibility but different stochastic representations and representation-dependent predictive trade-offs.

\section{Related Work}

\textbf{Projection and correction.}
A major branch predicts in an ambient space and enforces feasibility through
differentiable optimization, completion/correction, radial maps, or projection
layers
\citep{amos2017,agrawal2019,donti2021,liang2023,liang2024,schneider2026}.
ProbHardE2E~\citep{utkarsh2025} is the closest dedicated probabilistic method in
this branch: it constructs hard-constrained predictive laws through differentiable
probabilistic projection.

\textbf{Direct feasible parameterization.}
Another branch constructs outputs directly in feasible coordinates, including
null-space and inequality parameterizations, convex feasible layers, and hard-constrained network architectures
\citep{frerix2020,tordesillas2023,min2024,nguyen2025}.
A domain-specific probabilistic precedent is SMT-HCSD
\citep{bekele2026smthcsd}, which generates stochastically ordered temperature outputs using a nonnegative increment.

\textbf{Structural and probabilistic constrained laws.}
Several works are particularly close to the probabilistic structural viewpoint.
CLOVER~\citep{olivares2024clover} constructs coherent probabilistic forecasts by
construction through a differentiable structural factor representation for hierarchical forecasting.
Girolimetto and Di Fonzo~\citep{girolimetto2024general} develop a
structural-like formulation that separates free and constrained variables for
general linearly constrained multiple time series and extend it to probabilistic
forecast reconciliation. \citet{li2026hardlinear} develop
probabilistically sound deep generative learning under hard linear equality
constraints.

\textbf{Position of \SLPProbHard{}.}
These precedents address hierarchical forecasting (CLOVER), general linear time-series constraints (Girolimetto--Di Fonzo), hard linear generative learning (Li et al.), and stochastic temperature ordering (SMT-HCSD). The present study brings these structural-probabilistic perspectives into a common representation audit spanning affine, ordering, simplex, and explicit nonlinear geometries. \SLPProbHard{} is organized around a complementary cross-family question: \emph{when the feasible structural representation itself is chosen as the
coordinate system of the predictive law, how does that choice affect support and
boundary probability, stochastic-coordinate structure, uncertainty propagation,
expressiveness/calibration, and computation?}
Section~\ref{sec:results} gives the compatible official-source
ProbHardE2E/DPPL affine-adapter comparison, while
Appendix~\ref{app:related} provides the broader mechanism-level comparison.

\begin{figure}[t]
\centering
\resizebox{\linewidth}{!}{%
\begin{tikzpicture}[
  font=\scriptsize,
  box/.style={
    draw=black!55,
    rounded corners=1.4pt,
    align=center,
    inner sep=2.0pt,
    minimum height=5.7mm,
    fill=white
  },
  method/.style={
    font=\bfseries\scriptsize,
    anchor=east,
    align=right,
    text width=20mm
  },
  corr/.style={box,fill=red!5},
  proj/.style={box,fill=green!7},
  probproj/.style={box,fill=blue!7},
  slp/.style={box,fill=violet!7},
  outbox/.style={box,fill=black!3},
  scopebox/.style={
    box,
    fill=black!2,
    font=\tiny,
    minimum width=31mm,
    minimum height=14mm
  },
  arr/.style={
    -{Latex[length=1.5mm]},
    thick,
    draw=black!75
  }
]

\node[font=\scriptsize\bfseries,anchor=west]
at (0,0.55)
{A. Where hard feasibility enters the predictive model};

\node[method] (m1) at (0.0,-0.05) {DC3};
\node[corr,right=1.8mm of m1] (d1)
{predicted\\ output};
\node[corr,right=2.2mm of d1] (d2)
{completion +\\ correction};
\node[outbox,right=2.2mm of d2] (d3)
{$y\in\Cset$};

\draw[arr] (d1) -- (d2);
\draw[arr] (d2) -- (d3);

\node[font=\tiny,anchor=west,right=2mm of d3]
{deterministic correction};

\node[method] (m2) at (0.0,-0.86) {$\Pi$net};
\node[proj,right=1.8mm of m2] (p1)
{network\\ output};
\node[proj,right=2.2mm of p1] (p2)
{orthogonal\\ projection layer};
\node[outbox,right=2.2mm of p2] (p3)
{$y\in\Cset$};

\draw[arr] (p1) -- (p2);
\draw[arr] (p2) -- (p3);

\node[font=\tiny,anchor=west,right=2mm of p3]
{deterministic projection};

\node[method] (m3) at (0.0,-1.67) {ProbHardE2E};
\node[probproj,right=1.8mm of m3] (h1)
{ambient law\\ $P_\theta(\widetilde{Y}\mid x)$};
\node[probproj,right=2.2mm of h1] (h2)
{DPPL\\ probabilistic projection};
\node[outbox,right=2.2mm of h2] (h3)
{feasible law\\ $P_\theta(Y\mid x)$};

\draw[arr] (h1) -- (h2);
\draw[arr] (h2) -- (h3);

\node[font=\tiny,anchor=west,right=2mm of h3]
{projected predictive law};

\node[method] (m4) at (0.0,-2.85) {\SLPProbHard{}};
\node[slp,right=1.8mm of m4] (s1)
{latent law\\ $P_\theta^Z(Z\mid x)$};
\node[slp,right=2.2mm of s1] (s2)
{Feasible map\\ $Y=h_\phi(x,Z)$};
\node[outbox,right=2.2mm of s2] (s3)
{feasible law\\ $P_\theta^Y$};

\draw[arr] (s1) -- (s2);
\draw[arr] (s2) -- (s3);

\node[font=\tiny,anchor=west,right=2mm of s3]
{structural feasible pushforward};

\node[
  draw=violet!70!black,
  rounded corners=2pt,
  thick,
  inner sep=1mm,
  fit=(s1)(s2),
  label={
    [font=\tiny,text=violet!70!black,fill=white,inner sep=1pt]
    above:{$\mathrm{SFLP}=(P_\theta^Z,h_\phi)$}
  }
] (sflpgroup) {};

\node[font=\tiny,align=center,below=1.1mm of s2]
{same hard-feasibility requirement; different modeling object and
induced predictive law};

\node[font=\scriptsize\bfseries,anchor=west]
at (0,-3.95)
{B. Constraint structures and real-world representations evaluated here};

\node[scopebox,anchor=north west] (b1) at (0,-4.15)
{\textbf{Linear equality}\\
 Synthetic: $Ay=b$\\
 Real-world: affine FDC\\
 Hard ref.: Projection-Or\\
 + \textbf{DPPL adapter}};

\node[scopebox,right=2.2mm of b1] (b2)
{\textbf{Convex structural}\\
 Synthetic: order + simplex\\
 Real-world: nonnegative FDC\\
 Hard refs.: isotonic /\\
 simplex projection};

\node[scopebox,right=2.2mm of b2] (b3)
{\textbf{Explicit nonlinear}\\
 Synthetic: unit circle\\
 Real-world: FDC scale-shape\\
 Hard refs.: radial /\\
 alternating correction};

\end{tikzpicture}%
}

\caption{\textbf{\SLPProbHard{} positioning and empirical scope.}
\textbf{Panel A} contrasts representative locations at which hard
feasibility enters learning: DC3 uses completion/correction~\citep{donti2021},
$\Pi$net uses orthogonal projection layers~\citep{grontas2025}.
At the probabilistic level, ProbHardE2E/DPPL and \SLPProbHard{}/\SFLP{} are parallel hard-constraint paradigms: DPPL constructs a feasible law by projecting an ambient predictive law~\citep{utkarsh2025}, whereas an \SFLP{} defines the feasible law directly as the pushforward of its structural latent law. DC3 and $\Pi$net are shown for conceptual positioning
only; the direct official-source empirical comparison in this paper is
with ProbHardE2E/DPPL on the compatible affine benchmark.
\textbf{Panel B} summarizes the representative linear, convex-structural,
and explicit nonlinear geometries evaluated synthetically and through
three structural views of the same hydrological FDC data.}

\label{fig:concept}
\end{figure}
\section{\SLPProbHard{}}
\subsection{Structural feasible predictive laws}

Let $x\in\mathcal X$, $y\in\mathcal Y\subseteq\mathbb R^d$, let $\mathcal Z$ be the latent space, and let $\mathcal C(x)$ be the sample-level feasible set. Within \SLPProbHard{}, an \SFLP{} is the pair formed by a structural latent predictive law and a feasible structural map:
\[
Z\sim P_\theta^Z(\cdot\mid x),\qquad Y=h_\phi(x,Z),\qquad
h_\phi(x,z)\in\mathcal C(x)\quad\forall z\in\mathcal Z.
\]
The induced predictive law is $P_\theta^Y=h_{\phi,x}\push P_\theta^Z$. Together, the latent law and map determine the family of feasible predictive distributions that can be represented. In all reported experiments, $h_\phi$ is fixed by known constraint geometry; learnable structural maps are not evaluated here.

When an output density is inconvenient, proper sample-based scores provide a generic training interface~\citep{gneiting2007,jordan2019}. Reparameterized sampling permits samplewise backpropagation; interpreting this as a gradient of the expected score requires the regularity conditions discussed in Appendix~\ref{app:pathwise}~\citep{kingma2014,figurnov2018}.

For affine equalities $Ay=b$, choose a feasible point $y_0$ and full null-space basis $N\in\R^{d\times q}$ with $q=d-m$, $m=\operatorname{rank}(A)$, $Ay_0=b$, and $AN=0$:
\begin{equation}
Y=y_0+NZ,\qquad Z\in\R^q.\label{eq:null}
\end{equation}
Here $q$ is both the intrinsic and computational stochastic dimension. Other maps may be redundant, so we keep intrinsic degrees of freedom separate from computational latent dimension.

\begin{table}[t]
\centering\tiny
\setlength{\tabcolsep}{2.1pt}
\renewcommand{\arraystretch}{1.08}
\begin{tabularx}{\linewidth}{p{0.14\linewidth}p{0.18\linewidth}p{0.18\linewidth}p{0.19\linewidth}Y}
\toprule
\textbf{Tested geometry} & \textbf{Synthetic benchmark} & \textbf{Real-world representation} & \textbf{Primary SFLP} & \textbf{Hard comparator}\\
\midrule
Linear equality & hierarchy $Ay=b$ & affine FDC coherence & $y_0+Nz$; affine $(Q,\operatorname{diff}Q)$ & orthogonal projection; \textbf{DPPL official-source adapter} (synthetic only)\\
Convex structural instances & weak order; simplex & nonnegative FDC ordering & cumulative / normalized positive-part & isotonic / simplex projection\\
Explicit nonlinear equality/manifold & unit circle & FDC scale-shape $\Delta_i=RP_i$ & angle / composed scale-shape & radial projection / alternating feasible correction\\
\bottomrule
\end{tabularx}
\caption{Empirical constraint coverage. These experiments evaluate \SLPProbHard{} on representative structures from the listed geometry groups. The direct official-source DPPL comparison is restricted to the synthetic affine benchmark.}
\label{tab:maps-main}
\end{table}

\begin{center}
\fbox{\parbox{0.95\linewidth}{
\textbf{Algorithm 1: \SLPProbHard{} - design, train, and generate.}

\textbf{A. SFLP design/audit:}\\
(1) specify the sample-level feasible set $\Cset(x)$;\\
(2) choose structural coordinates and $h_\phi$;\\
(3) audit the \SFLP{}: verify $h_\phi(x,\Zset)\subseteq\Cset(x)$ and record
(i) map image and coverage/surjectivity limitations;
(ii) boundary preimages and whether they have zero or positive probability under the selected latent law;
(iii) intrinsic dimension $q_{\rm int}$ and computational latent dimension $q_{\rm comp}$;
(iv) sensitivity, redundancy, and induced dependence/covariance restrictions;
(v) density/Jacobian tractability where applicable; and
(vi) whether generation requires a separate optimization-based projection or correction stage.

\textbf{B. Probabilistic training:}\\
(4) predict latent parameters $\eta_\theta(x)$;\\
(5) draw reparameterized
$z^{(s)}\sim P^Z_\theta(\cdot\mid x)$;\\
(6) construct
$\widehat y^{(s)}=h_\phi(x,z^{(s)})$;\\
(7) compute the selected score or likelihood and its training gradient,
with the estimator qualifications in Appendix~\ref{app:pathwise}.\\
\textbf{C. Inference:}\\
(8) sample $Z$;\\
(9) apply $h_\phi$ and return feasible predictive samples without a
separate correction stage.

\textit{Scope:} If no useful explicit \SFLP{} is available,
\SLPProbHard{} does not assume that structural parameterization is
preferable to projection/correction.
}}
\end{center}
This audit makes representation choice operational: feasible constructions can differ in coverage, boundary probability, stochastic coordinates, induced law families, and computation.
\subsection{Representation-induced probabilistic consequences}
The \SFLP{} audit exposes probabilistic consequences beyond sample-level feasibility, using three standard mathematical facts. First, if $h_x(\Zset)\subseteq\Cset(x)$, then the pushforward satisfies $P^Y_\theta(\Cset(x)\mid x)=1$. Second, rank-nullity gives $q=d-m$ free coordinates for a consistent affine system. Third, for any measurable event $B$,
\begin{equation}
P^Y(B)=P^Z\!\left(h^{-1}(B)\right).\label{eq:boundary}
\end{equation}

\paragraph{Two representation-level consequences.}
Equation~\eqref{eq:boundary} makes boundary probability representation-dependent:
a scientifically meaningful boundary event can be excluded, reachable with zero
probability, or receive positive predictive mass according to its latent
preimage. A second axis is stochastic representation: we distinguish the
intrinsic degrees of freedom $q_{\mathrm{int}}$ from the computational latent
dimension $q_{\mathrm{comp}}$. They coincide for affine null-space coordinates
and the circle angle map, need not decrease for ordering, and differ in the
boundary-capable simplex and nonlinear FDC implementations. Dimensional economy
is therefore map-dependent rather than a universal property of
\SLPProbHard{}.

The representation-level probabilistic analysis asks what a chosen map does \emph{beyond} feasibility. Using standard truncated-Gaussian calculations~\citep{tallis1961}, for $R=\pospart{Z}$ and $Z\sim\mathcal N(\mu,\sigma^2)$, with $\sigma>0$ and $\alpha=\mu/\sigma$,
\begin{equation}
\Pr(R=0)=\Phi(-\alpha),\qquad
\E[R]=\sigma\varphi(\alpha)+\mu\Phi(\alpha).\label{eq:main-pospart}
\end{equation}
For the following product formulas, assume conditionally independent
$Z_i\mid x\sim\mathcal N(\mu_i(x),\sigma_i^2(x))$ with $\sigma_i(x)>0$;
conditioning on $x$ is suppressed below.
The cumulative weak-order map gives
$\Pr(Y_k=Y_{k-1}\mid x)=\Phi(-\mu_k/\sigma_k)$ for $k\ge2$, while nonnegative
cumulative ordering gives
$\Pr(Y_k=0\mid x)=\prod_{j\le k}\Phi(-\mu_j/\sigma_j)$.
For the normalized positive-part simplex map, define the random active
set $\mathcal A=\{i:Z_i>0\}$. For any nonempty index set $S$,
\begin{equation}
\Pr(\mathcal A=S\mid x)
=
\prod_{i\in S}\Phi(\mu_i/\sigma_i)
\prod_{j\notin S}\Phi(-\mu_j/\sigma_j).\label{eq:main-simplex-active}
\end{equation}
Conditional on $\mathcal A=S\ne\varnothing$, the output lies in the
relative interior of the face indexed by $S$. The all-inactive event
is assigned to a vertex by the fallback in Appendix~\ref{app:simplex-map}, adding mass
to vertex outcomes.
 For $Z\sim\mathcal N(\mu,\sigma^2)$ and $Y=(\cos Z,\sin Z)$, $\E[\cos Z]=e^{-\sigma^2/2}\cos\mu$ and $\E[\sin Z]=e^{-\sigma^2/2}\sin\mu$. These formulas, their second moments, and covariance expressions are collected in Appendix~\ref{app:dist-propagation}; here they expose representation-specific probability allocation. The same analysis records whether the stochastic coordinates are intrinsic (affine, circle), full-dimensional (ordering), or computationally redundant (normalized simplex and nonlinear FDC).

\subsection{Structural-map design considerations}
Structural-map choice introduces trade-offs beyond feasibility. We use three practical design considerations to characterize these choices; they are interpretive considerations rather than a universal optimization criterion. Coherence asks whether the map enforces the full stated structure jointly, sample by sample. Sensitivity describes how latent variation propagates through the map, including Jacobian structure, flat regions, saturation, and boundary atoms. Simplicity concerns explicitness, latent redundancy, and whether an additional correction or solve stage is necessary. For the positive-part map $h(z)=\pospart{z}$, $h'(z)=\mathbf 1_{\{z>0\}}$ almost everywhere, so under $Z\sim\mathcal N(\mu,\sigma^2)$,
\begin{equation}
\E[h'(Z)]=\Pr(Z>0)=\Phi(\mu/\sigma)=1-\Pr\!\left(h(Z)=0\right).
\label{eq:sensitivity-boundary}
\end{equation}
Thus the same latent parameters that control exact boundary mass also control the expected active-gradient probability for this map. This identity is a sensitivity diagnostic, not an optimization guarantee. These considerations explain why two equally feasible maps can still differ statistically or computationally.

\subsection{Relationship to projection and ProbHardE2E}
ProbHardE2E/DPPL constructs a feasible law by projecting an ambient predictive law, whereas an \SFLP{} defines it directly through a structural latent law and feasible map. For affine $AY=b$, DPPL conditions an ambient diagonal Gaussian~\citep{utkarsh2025}; \SLPProbHard{} models the $d-m$ free coordinates in Eq.~\eqref{eq:null}. The DPPL conditioning and sampling equations are given in Appendix~\ref{app:dppl-equations}.

The constructions can have identical feasible support yet admit different covariance families under the selected heads. Equality of support therefore does not imply equality of probability allocation. Projection remains useful for implicit or difficult feasible sets; structural generation is attractive when meaningful feasible coordinates exist or boundary behavior matters. Appendix~\ref{app:comparison} provides the detailed mechanism comparison.

\section{Experiments}
We evaluate frozen held-out synthetic and seven-basin FDC benchmarks against unconstrained, soft, post-hoc, and matched end-to-end hard baselines; Appendix~\ref{app:protocol} provides the complete seed registries and protocols. The affine benchmark additionally includes ProbHardE2E/DPPL~\citep{utkarsh2025} using the official released repository, frozen at commit \texttt{ff787718e79e5ac6b673807d40ce1cddb0fd17c2}, and a verified affine adapter; the existing \SLPProbHard{} seed-20-29 results were reused unchanged rather than retrained.

Synthetic comparisons use matched backbones and training protocols, optimizing an equally weighted combination of marginal CRPS and Energy Score. The official-source DPPL~\citep{utkarsh2025,amazonprobharde2e2026}  affine adapter preserves the released Gaussian conditioning equations and passes an equation-level equivalence audit. Appendix~\ref{app:protocol} gives the training and evaluation settings; Appendix~\ref{app:dppl-audit} documents the adapter and verification.

The real-world source is a single frozen seasonal FDC dataset built from daily river discharge from the \emph{GloFAS v5.0 hydrological reanalysis}, a LISFLOOD-based JRC dataset~\citep{grimaldi2026glofas,vanderknijff2010,burek2013}, for seven Emilia-Romagna basins. Basin-aggregated precipitation, minimum-temperature, and maximum-temperature descriptors are derived from ERACLITO61~\citep{antolini2016}, with season, basin area, and basin identity completing the 22 predictors. The common 1996-2022 period yields 756 samples with a 532/112/112 temporal train/validation/test split. Following standard FDC/exceedance-probability representations and multi-output FDC learning~\citep{burgan2018,worland2019,vogel1994,vogel1995}, the ordered FDC target is
\begin{equation}
Q=[Q_{95},Q_{90},Q_{75},Q_{50},Q_{25},Q_{10},Q_{05}],\qquad 0\le Q_{95}\le\cdots\le Q_{05}.\label{eq:fdc}
\end{equation}
The three targets provide scientifically interpretable representations of the same underlying FDC. The seven levels describe discharge across exceedance regimes. Adjacent increments quantify changes between neighboring regimes, while the scale--shape representation separates the discharge range $R=Q_{05}-Q_{95}$ from its relative allocation across the curve. For $R>0$, normalized increments $P_i=\Delta_i/R$ describe this allocation; Appendix~\ref{app:fdc-nonlinear} specifies the flat-curve convention. The affine representation exposes increment coherence, whereas the nonlinear representation separates the low-flow baseline, discharge range, and relative increment allocation. Reusing the underlying FDCs and split allows the experiments to examine how these representation choices affect predictive laws.

The original nonnegative-ordering benchmark models $\log(1+Q)$. The affine and nonlinear extensions reuse the \emph{same frozen predictors, raw FDCs, basin/season/year identities, and split} without rereading daily discharge or recomputing FDCs; they apply one train-only multiplicative RMS scale $s_Q=43.7191$ to raw $Q$, preserving difference and scale-shape identities exactly. Cross-representation comparisons therefore concern the combined representation and preprocessing choices; they do not isolate a change of feasible map under identical target preprocessing. Appendix~\ref{app:protocol} gives the complete constructions, seed registries, baselines, metrics, and reproducibility details.

\section{Results}\label{sec:results}
\subsection{Primary synthetic benchmarks}
Table~\ref{tab:primary-compact} reports matched end-to-end hard references, the direct official-source ProbHardE2E/DPPL affine-adapter comparison, and SLP. The affine comparison first highlights the stochastic representation: DPPL conditions an ambient Gaussian with 11 stochastic coordinates, whereas \SLPProbHard{} learns an 8-dimensional Gaussian in affine free coordinates, a 27.27\% reduction. Both have zero observed violations at the $10^{-5}$ threshold, but their primary diagonal heads induce different feasible covariance families. \SLPProbHard{} lowers MSE by 1.30\% ($p=0.0072$) and MAE by 0.43\% ($p=0.0353$), with 2.48\% narrower intervals ($p=0.0341$). DPPL lowers marginal CRPS by 1.18\% ($p=3.76\times10^{-4}$) and gives higher coverage (0.809 versus 0.786), closer to the nominal 0.90 level. The paired test does not detect an Energy Score difference ($p=0.391$). These results are consistent with representation-dependent probability allocation within the same feasible affine support. Projection-Or remains a transparent closed-form reference; richer \SLPProbHard{} covariance later closes or reverses its diagonal-head score gap. In ordering and simplex, positive-part \SLPProbHard{} significantly improves the reported CRPS/ES relative to the matched end-to-end projection.

On the unit circle, the angle map $h(z)=(\cos z,\sin z)$ achieves zero observed $10^{-5}$ violations using one stochastic coordinate versus two for radial projection. Paired tests detect no MSE, CRPS, or Energy Score differences across ten seeds. Appendix~\ref{app:additional-results}, under ``Nonlinear circle primary results,'' gives the full comparison.

\begin{table}[t]
\centering\scriptsize
\resizebox{\linewidth}{!}{%
\begin{tabular}{llrrrrrr}
\toprule
Family & Method & MSE $\downarrow$ & CRPS $\downarrow$ & ES $\downarrow$ & Cov.90 & Stoch. dim. & Feasibility\\
\midrule
Linear & Projection-Or & 0.2965 & 0.2847 & 1.2088 & 0.785 & 11 & VR$_{10^{-5}}=0$\\
Linear & DPPL (official-source adapter) & 0.2997 & 0.2868 & 1.2190 & 0.809 & 11 & VR$_{10^{-5}}=0$\\
Linear & \SLPProbHard{} & 0.2957 & 0.2901 & 1.2216 & 0.786 & 8 & VR$_{10^{-5}}=0$\\
Order & Projection-isotonic & 5.5677 & 0.7945 & 2.2440 & 0.784 & 5 & exact order\\
Order & SLP-positive-part & 5.0948 & 0.7820 & 2.1883 & 0.797 & 5 & exact order\\
Simplex & Projection-simplex & 0.0322 & 0.0890 & 0.2474 & 0.851 & 5 & CE $1.51\!\times\!10^{-8}$\\
Simplex & SLP-positive-part & 0.0320 & 0.0866 & 0.2431 & 0.846 & 5 & CE $2.30\!\times\!10^{-8}$\\
Circle & Radial projection & 0.3557 & 0.3277 & 0.5147 & 0.864 & 2 & VR$_{10^{-5}}=0$\\
Circle & SLP-angle & 0.3557 & 0.3272 & 0.5143 & 0.854 & 1 & VR$_{10^{-5}}=0$\\
\bottomrule
\end{tabular}}
\caption{Primary synthetic benchmarks (means across untouched seeds). ``DPPL (official-source adapter)'' denotes ProbHardE2E/DPPL using the official released source with the verified affine adapter; full source audit and paired statistics are in Appendix~\ref{app:dppl-audit}.}
\label{tab:primary-compact}
\end{table}

Boundary behavior illustrates Eq.~\eqref{eq:boundary}. For ordering, SLP's exact-tie rate 0.208 is closer than end-to-end isotonic projection's 0.396 to the observed 0.255. For the $d=5$ simplex, projection's zero rate 0.246 is closer to the observed 0.250 than SLP's 0.180. Thus boundary capability is not boundary calibration: the map determines whether an event can carry mass, while the latent law and training determine how much mass it receives.

\subsection{Composed and real-world structural representations}
A controlled synthetic nonnegative-ordering task generated targets by cumulative positive-part. Because the data-generating mechanism uses the same structure, this is mechanism validation for the boundary-preimage prediction rather than independent evidence that cumulative-positive-part \SLPProbHard{} generally dominates alternative enforcement. Both nonnegative isotonic projection and \SLPProbHard{} satisfy non-negativity and ordering exactly; predictive scores are close, with \SLPProbHard{} numerically lower on CRPS and Energy Score.

Holding the underlying FDC data and split fixed, Table~\ref{tab:hydro-structural} tests how three scientifically meaningful representations affect feasibility, stochastic dimensionality, calibration, and predictive scores. First, the nonnegative-ordering map directly enforces Eq.~\eqref{eq:fdc}; both hard methods eliminate negative discharge and crossings, with \SLPProbHard{} descriptively lower CRPS/ES but projection better calibrated in 90\% coverage. Second, the affine representation augments the seven FDC levels with six adjacent increments,
\begin{equation}
Y_{\mathrm{aff}}=(Q_0,\ldots,Q_6,\Delta_0,\ldots,\Delta_5)\in\R^{13},
\qquad \Delta_i=Q_{i+1}-Q_i,
\end{equation}
equivalently $Q_i-Q_{i+1}+\Delta_i=0$. The six equalities have rank 6, so the feasible affine representation has 7 intrinsic coordinates. \SLPProbHard{} uses $h(z)=(z,\operatorname{diff}z)$ with 7 stochastic coordinates versus 13 for ambient end-to-end affine projection, a 46.15\% reduction. Both have zero affine violations at $10^{-5}$; this branch enforces increment coherence only and does not additionally enforce nonnegativity or FDC ordering. The paired ten-seed tests did not detect differences in MSE, MAE, CRPS, or Energy Score; \SLPProbHard{} intervals are significantly narrower (10.94 vs.\ 13.57) but coverage is lower (0.712 vs.\ 0.746), so the result is a calibration-sharpness trade-off rather than a score win.

Third, for $R>0$, the nonlinear scale-shape representation defines
\begin{equation}
B=Q_0,\qquad \Delta_i=Q_{i+1}-Q_i,\qquad R=\sum_i\Delta_i,\qquad P_i=\Delta_i/R,
\end{equation}
and models
\begin{equation}
Y_{\mathrm{nl}}=(B,R,P,\Delta)\in\R^{14},
\qquad \Delta_i=RP_i,\quad P_i\ge0,\quad \sum_iP_i=1.\label{eq:fdc-nonlinear}
\end{equation}
\SLPProbHard{} uses $B=\pospart{z_B}$, $R=\pospart{z_R}$, normalized positive-part $P$, and $\Delta=RP$. For $d=7$, the mathematical intrinsic dimension is 7 but this boundary-capable implementation uses 8 computational stochastic coordinates; the ambient correction model uses 14, giving a 42.86\% reduction. Both coordinate reductions are relative to the corresponding augmented output representation. The original FDC has seven flow levels, and the nonlinear implementation retains one redundant computational coordinate.
Both end-to-end hard methods have zero nonlinear violations, zero negative reconstructed discharge, and zero ordering violations. None of the paired Raw-$Q$ MSE/MAE/CRPS/ES differences was statistically detected at the ten-seed level; \SLPProbHard{} is numerically lower on CRPS (4.047 vs.\ 4.177) and Energy Score (14.313 vs.\ 14.483). The comparator is a five-iteration differentiable alternating feasible correction, \emph{not} an exact joint Euclidean projection. Post-hoc feasible correction gives the lowest reported mean raw-$Q$ CRPS among the nonlinear baselines (3.568 versus 4.047 for SLP) and higher 90\% coverage (0.823 versus 0.649). Thus the nonlinear result supports feasible structural generation with fewer computational coordinates, while the strongest reported predictive scores favor post-hoc correction (Appendix~\ref{app:fdc-nonlinear-results}).

\begin{table}[t]
\centering\scriptsize
\resizebox{\linewidth}{!}{%
\begin{tabular}{llrrrrrr}
\toprule
Real-world representation & Hard method & Raw-$Q$ CRPS & Raw-$Q$ ES & Cov.90 & Width.90 & Stoch. dim. & Hard VR\\
\midrule
Nonneg. order & Projection-nonneg-isotonic & 3.894 & 15.561 & 0.757 & 14.15 & 7 & 0\\
 & SLP-nonneg-positive-part & 3.710 & 14.603 & 0.722 & 22.26 & 7 & 0\\
Affine coherence & End-to-end affine projection & 3.627 & 13.649 & 0.746 & 13.57 & 13 & 0\\
 & SLP-affine & 3.698 & 13.997 & 0.712 & 10.94 & 7 & 0\\
Nonlinear scale-shape & Post-hoc feasible correction & 3.568 & 13.754 & 0.823 & 14.97 & 14 & 0\\
 & End-to-end feasible correction & 4.177 & 14.483 & 0.630 & 15.75 & 14 & 0\\
 & SLP-scale-shape & 4.047 & 14.313 & 0.649 & 15.67 & 8 & 0\\
\bottomrule
\end{tabular}}
\caption{Three structural views of the same frozen seven-basin seasonal FDC dataset. ``Hard VR'' is the representation-specific $10^{-5}$ violation rate; the nonlinear hard methods also have zero negative reconstructed discharge and zero ordering violations. For the affine rows, \texttt{Hard VR} measures increment-coherence equalities only.
}
\label{tab:hydro-structural}
\end{table}
\subsection{Covariance and computation}
Richer \SLPProbHard{} covariance heads separate probabilistic expressiveness from structural feasibility. On the linear benchmark, moving from diagonal to rank-4 or full covariance improves \SLPProbHard{} CRPS from 0.2889 to about 0.2837 and Energy Score from 1.2171 to 1.2071 while keeping zero observed violations. This helps interpret the official-source DPPL affine-adapter result: covariance-weighted conditioning of an ambient diagonal Gaussian and a diagonal structural latent Gaussian induce different feasible covariance families, so identical support need not imply identical proper scores. This ablation varies covariance expressiveness within \SLPProbHard{} while preserving structural feasibility; it is not a fully covariance-matched DPPL comparison.

Under matched CPU instrumentation, median inference latency is 1.65 ms for \SLPProbHard{} versus 3.22 ms for the DPPL affine adapter. These implementation-specific timings and parameter counts are detailed in Appendix~\ref{app:profiling} and Appendix~\ref{app:dppl-audit}.
\subsection{Scaling, finite precision, and statistical interpretation}
The scaling experiments test whether the primary observations survive beyond a single dimension. For affine equalities, \SLPProbHard{} always uses $q=d-m$ stochastic coordinates across twelve pre-specified configurations, reducing stochastic dimension by 12.3-38.5\%. Point accuracy remains close to Projection-Or; the latter usually retains a small CRPS/ES advantage under the diagonal \SLPProbHard{} head. Re-evaluating the primary affine map in float64 collapses residuals from roughly $10^{-7}$ to $10^{-15}$, confirming that base-case residuals are numerical realization effects rather than relaxed mathematical constraints.

Zero thresholded violations in the primary benchmark do not extend to every float32 scaling configuration: at $(d,m,q)=(97,33,64)$, $\mathrm{VR}_{10^{-5}}$ is 0.0193 for SLP and 0.3958 for Projection-Or (Appendix~\ref{app:linear-scaling}). Ordering and simplex scaling show a different pattern. Positive-part \SLPProbHard{} has lower CRPS and Energy Score than the matched end-to-end projection in every evaluated dimension $d\in\{3,5,10,20,40\}$.

Complete paired confidence intervals, Wilcoxon tests, source audits, scaling tables, and numerical diagnostics are given in the appendix.

\section{Discussion and Limitations}
Exact feasibility does not uniquely specify a predictive law.
The experiments show how structural representation shapes boundary
probability, stochastic coordinates, attainable covariance, and
calibration-sharpness trade-offs. Coordinate savings depend on the
selected map: affine and circle constructions exploit intrinsic
coordinates, whereas normalized simplex and nonlinear scale-shape
retain computational redundancy.
\paragraph{From constraint sets to structural representations.} The three FDC views test physical ordering, increment coherence, and scale--shape coherence on one fixed dataset and split. Together with their distinct preprocessing, they show how scientifically meaningful representation choices affect probabilistic predictions. Their hydrological interpretation is developed in Appendices~\ref{app:protocol} and~\ref{app:discussion-representation}.

\paragraph{Calibration.}
Hard feasibility does not ensure calibrated uncertainty. The reported
90\% marginal coverage for SLP is 0.722, 0.712, and 0.649 across the
three FDC representations. These values show substantial undercoverage;
interval width must therefore be interpreted jointly with coverage,
and narrower intervals alone do not establish better uncertainty
quantification.

\paragraph{Explicit nonlinear-equality scope.} The unit circle and the FDC scale-shape map demonstrate two explicit nonlinear equalities, but they do not imply a general \SLPProbHard{} construction for arbitrary nonlinear constraints. The circle map is periodic and non-injective; the scale-shape map uses a boundary-capable simplex representation with one redundant computational coordinate. Projection/correction remains more general for implicit or difficult feasible sets.
\paragraph{Computation.} Stochastic-coordinate savings are architectural and hardware independent. Wall-clock improvements depend on the workload and implementation; optimized affine conditioning is itself inexpensive. Appendix~\ref{app:discussion-computation} distinguishes these claims.
\paragraph{Limitations.} \SLPProbHard{} requires a useful explicit map, and non-surjective maps introduce representation bias. The primary positive-part boundary-mass implementation is nonsmooth and has a flat negative region. For maps with discontinuous fallbacks, samplewise differentiation alone does not establish unbiased gradients of expected scores. Nonlinear pushforwards may lack tractable densities or moments. The hydrological target is GloFAS v5.0 LISFLOOD-based reanalysis discharge rather than observed-gauge discharge~\citep{grimaldi2026glofas}. The nonlinear FDC comparator is an alternating feasible correction rather than an exact joint Euclidean projection. The direct ProbHardE2E experiment is an official-source \emph{affine-adapter} comparison, not a reproduction of the authors' PDE datasets or published wall-clock experiments; it establishes one compatible direct comparison, not superiority across ProbHardE2E's broader supported classes. Additional limitations and numerical diagnostics are in the appendix.

\section{Conclusion}
\SLPProbHard{} introduces a cross-family representation-centered framework in which a structural latent law and a feasible map jointly define a hard-feasible predictive law. The latent law and map jointly determine predictive support and boundary probability; the chosen coordinates and map specify the stochastic representation and shape attainable dependence, calibration, expressiveness, and computation. Experiments span affine equalities, ordering, simplex, explicit nonlinear maps, and three representations of the same hydrological FDC data. The direct affine DPPL comparison demonstrates shared practical feasibility and representation-dependent predictive trade-offs. Structural generation is useful when suitable explicit coordinates exist; projection and correction remain useful for implicit or difficult feasible sets.

\clearpage
\section*{Acknowledgments}
This study and the related research have been carried out within, and with the support of, the Italian inter-university PhD Programme in Sustainable Development and Climate Change (PhD-SDC; www.phd-sdc.it), funded by the European Union – NextGenerationEU under PNRR, at IUSS University School for Advanced Studies Pavia, with the University of Parma as the host institution. 

The simulations presented in this work were performed on the High Performance Computing Data Center at IUSS under the PhD-SDC project. The computing facility was co-funded by Regione Lombardia through the funding programme established by Regional Decree No. 3776 of November 3, 2020.

\section *{Funding}
This study received no dedicated project funding. Doctoral and computational support are acknowledged above.

\bibliography{references}
\bibliographystyle{slp_preprint}
\clearpage
\appendix

\section{Extended Related Work}\label{app:related}
\begin{table}[H]
\centering\tiny
\setlength{\tabcolsep}{2pt}
\renewcommand{\arraystretch}{1.08}
\begin{tabularx}{\linewidth}{p{0.24\linewidth}Yp{0.12\linewidth}p{0.14\linewidth}Yp{0.18\linewidth}}
\toprule
\textbf{Work} &
\textbf{Core mechanism} &
\textbf{Prob. law?} &
\textbf{Hard feasibility} &
\textbf{Representation emphasis} \\
\midrule

OptNet / diff.\ convex
\citep{amos2017,agrawal2019}
&
optimization layer
&
opt.-centered
&
supported problem
&
optimization geometry
\\

DC3~\citep{donti2021}
&
completion + correction
&
det./opt.
&
correction-dependent
&
completion/correction
\\

\citet{frerix2020}
&
direct linear-inequality map
&
deterministic
&
yes, supported class
&
feasible geometry
\\

RAYEN~\citep{tordesillas2023}
&
direct convex feasible layer
&
deterministic
&
yes, convex class
&
convex parameterization
\\

CLOVER~\citep{olivares2024clover}
&
differentiable coherent structural generation
&
yes
&
yes, hierarchy
&
coherent probabilistic representation
\\

Girolimetto-Di Fonzo~\citep{girolimetto2024general}
&
free/constrained structural-like representation
&
yes
&
yes, linear coherence
&
general linear structural coordinates
\\

SMT-HCSD~\citep{bekele2026smthcsd}
&
stochastic nonnegative increment
&
yes
&
yes, ordering
&
domain-specific structural generation
\\

ProbHardE2E~\citep{utkarsh2025}
&
probabilistic projection
&
yes
&
yes, supported classes
&
projected-law formulation
\\

\citet{li2026hardlinear}
&
hard linear constrained generative law
&
yes
&
yes, linear equality
&
probabilistic constrained linear law
\\

Soft-Radial~\citep{schneider2026}
&
radial feasible map
&
map-centered
&
interior feasible
&
interior support geometry
\\

CNP~\citep{hintermuller2026cnp}
&
admissible neural coordinates
&
opt.-centered
&
yes, supported sets
&
function-space coordinates
\\

\SLPProbHard{}
&
structural latent pushforward
&
yes
&
yes, map image
&
support/boundary, stochastic coordinates,
expressiveness/calibration, computation
\\

\bottomrule
\end{tabularx}

\caption{\textbf{Literature comparison by mechanism and probabilistic emphasis.}
Structural probabilistic generation has important precedents in hierarchical,
linear, and domain-specific settings. \SLPProbHard{} uses structural
representation as a cross-family probabilistic model-design object and jointly
analyzes its consequences for support/boundary probability, stochastic
coordinates, expressiveness/calibration, and computation.
The direct ProbHardE2E comparison in this paper is restricted to the compatible
synthetic affine benchmark.}
\label{tab:literature}
\end{table}

\subsection{Hard constraints in machine learning}
Hard constraints define a feasible set $\Cset$ through equalities, inequalities, order relations, conservation laws, or boundary conditions. Soft penalties can be useful but leave a penalty-weight trade-off and cannot in general guarantee zero violation. Differentiable optimization layers such as OptNet and differentiable convex optimization layers bring exact or approximate optimization into end-to-end learning~\citep{amos2017,agrawal2019}. Classical convex projection geometry is summarized by \citet{boyd2004}. For constraint families used in this paper, efficient Euclidean projection algorithms are well established for simplex/$\ell_1$ geometry~\citep{duchi2008,condat2016}, while isotonic methods provide classical projection machinery for order cones~\citep{best1990,deleeuw2009}. Differentiable relaxations and projections for sorting/ranking provide related trainable machinery~\citep{cuturi2019,blondel2020}. More recent hard-constraint layers span operator-splitting projection, LMI projection, nonlinear repair, learnable slack-variable projection, and hard projection inside physics-informed networks~\citep{grontas2025,tang2026,goertzen2026,chu2026,wang2026diffslack,horne2026}.

\subsection{Probabilistic forecasting under constraints}
In probabilistic prediction, hard feasibility should hold at the sample level, not only for a conditional mean. Proper scoring rules such as CRPS and Energy Score evaluate distributional quality through calibration and sharpness~\citep{matheson1976,hersbach2000,gneiting2007,gneiting2008,gneiting2014,jordan2019}; energy-distance ideas underlie multivariate distance-based evaluation~\citep{szekely2013}, while variogram scores provide a dependence-sensitive alternative~\citep{scheuerer2015}. Coherent forecasting developed from reconciliation of hierarchical point forecasts~\citep{hyndman2011,hyndman2016} to explicitly probabilistic reconciliation and end-to-end coherent prediction~\citep{bentaieb2017,rangapuram2021,rangapuram2023,jeon2019,wickramasuriya2019,panagiotelis2023,girolimetto2024}. ProbHardE2E~\citep{utkarsh2025} provides a recent explicit probabilistic hard-constraint framework based on differentiable projection. \SLPProbHard{} pursues the same high-level objective but makes the feasible representation generative rather than corrective.

\subsection{Feasible-by-construction parameterization}
Many constraints admit direct parameterizations: nonnegative structural maps for non-negativity, cumulative nonnegative increments for ordering, null-space coordinates for linear equalities, normalized nonnegative weights for the simplex, and ansatz functions for initial or boundary conditions. Related monotonic-network constructions use lattice, integral, certified, constrained-weight, and smooth min-max mechanisms to encode monotonicity~\citep{you2017,wehenkel2019,liu2020,runje2023,igel2024}. For compositional outputs, the simplex has a long statistical literature built around logistic-normal and compositional representations~\citep{aitchison1980,aitchison1982}; sparse probability mappings such as sparsemax illustrate that exact zeros can be an intentional representation choice~\citep{martins2016}. Exact-condition ansatzes likewise have a long constructive lineage in neural differential-equation solvers~\citep{lagaris1998}, with later geometry-aware distance-function formulations and comparative PINN studies imposing Dirichlet conditions exactly~\citep{sukumar2022,berrone2023}. These precedents motivate structural maps as established primitives; the present paper studies their role inside a probabilistic predictive-law methodology.

\subsection{Constructive feasibility lineage}
\SLPProbHard{} sits within a broader constructive-feasibility lineage. DC3 combines equality completion with inequality correction~\citep{donti2021}. \citet{frerix2020} parameterize activations satisfying homogeneous linear inequalities. Homeomorphic Projection maps from a simple domain into a constrained feasible set~\citep{liang2023,liang2024}. RAYEN enforces hard convex constraints without expensive orthogonal projection or inner gradient descent~\citep{tordesillas2023}. HardNet studies hard-constrained networks with universal approximation guarantees~\citep{min2024}, while FSNet integrates feasibility-seeking into constrained optimization learning~\citep{nguyen2025}. Soft-Radial Projection maps ambient points radially into the interior of a convex feasible set and is particularly relevant as a recent non-orthogonal feasibility transformation~\citep{schneider2026}. The rapidly developing projection/repair branch also includes $\Pi$net for differentiable orthogonal projection, LMI-Net for semidefinite constraints, HardNet++ and SnareNet for nonlinear feasibility, DiffSlack for nonlinear inequalities, and KKT-based hard-constraint PINNs~\citep{grontas2025,tang2026,goertzen2026,chu2026,wang2026diffslack,chen2024kktpin,iftakher2025kkthardnet}.

Most of these methods target deterministic feasible outputs or optimization solutions. \SLPProbHard{} instead centers a full predictive distribution: stochastic latent samples are mapped into feasible outputs, and the resulting pushforward law is trained by likelihoods or proper scoring rules.

\paragraph{Adjacent probabilistic and feasible-parameterization work.} \citet{li2026hardlinear} study probabilistically sound deep generative modeling under hard linear equalities, directly establishing that hard-constrained generation and probability modeling can be combined in the linear setting. Soft-Radial Projection~\citep{schneider2026} constructs a differentiable radial reparameterization into the strict interior of convex feasible sets, while Constrained Neural Parameterization~\citep{hintermuller2026cnp} maps unconstrained neural coordinates into admissible function-space sets for optimization. These papers make the novelty boundary important: \SLPProbHard{} does not claim that feasible-by-construction transforms are new. Its methodological contribution is the coupled probabilistic treatment of structural latent generation, map-induced support and boundary atoms, free versus redundant stochastic coordinates, conditional approximation/parameterization bias, and cross-family proper-score/computational consequences under a single hard-constrained predictive-law formalism.

\paragraph{Additional context.} The constructive-feasibility landscape also includes differentiable sorting/ranking, isotonic regression, monotone networks, exact-condition ansatzes, semidefinite-output layers, nonlinear repair, KKT-based formulations, and structured probabilistic reconciliation. These threads motivate the broader taxonomy but differ in whether they target deterministic feasibility, constrained optimization, or a complete predictive distribution. Together, these adjacent lines motivate the paper's representation-centered view of probabilistic hard-constrained learning.

\section{Complete Method and Structural-Map Taxonomy}\label{app:maps}
\subsection{Structural latent distribution}
A neural backbone $f_\theta$ predicts latent distribution parameters $\eta_\theta(x)$. Under a reparameterized Gaussian,
\begin{equation}
Z=\mu_\theta(x)+L_\theta(x)\epsilon,\qquad \epsilon\sim\mathcal N(0,I).
\end{equation}
The framework is not restricted to Gaussian latents; Student-$t$, Gamma, LogNormal, zero-inflated, mixture, autoregressive, or flow-based families may be used when scientifically appropriate. Flexible invertible transports such as Real NVP and neural spline flows provide examples of richer latent-to-output distribution transforms~\citep{dinh2017,durkan2019,papamakarios2021}, while explicit and implicit pathwise-gradient methods broaden the trainable latent-family choices beyond simple location-scale Gaussians~\citep{figurnov2018,jankowiak2018}.

\subsection{Structural coordinates and computational latent dimension}
For $Ay=b$, choose a feasible point $y_0$ and a null-space basis $N\in\R^{d\times q}$ satisfying $Ay_0=b$ and $AN=0$. Then
\begin{equation}
Y=y_0+NZ,\qquad Z\in\R^q,
\end{equation}
uses exactly the free affine coordinates. An ambient projection model instead predicts $d$ quantities before removing infeasible directions.

This dimensional advantage is not universal. A convenient nonlinear map can use a redundant latent representation. In the normalized positive-part boundary-mass implementation below, the simplex has intrinsic dimension $d-1$ but the computational map uses $d$ latent coordinates. We therefore distinguish \emph{intrinsic degrees of freedom} from \emph{computational latent dimension}. Dimensional economy is a property of the chosen map, not of the \SLPProbHard{} label by itself.

\subsection{Feasibility map and pushforward law}
The feasibility map $h_\phi$ encodes the known structure. For a measurable set $B\subseteq\Yset$,
\begin{equation}
P^Y_\theta(B\mid x)=P^Z_\theta\bigl(h_{\phi,x}^{-1}(B)\mid x\bigr).
\end{equation}
Thus geometric properties of preimages under $h_\phi$ directly determine where the constrained predictive distribution can place probability.

\subsection{Training objectives}
When the induced density is tractable, one may minimize negative log-likelihood. Otherwise, sample-based proper scores provide a generic objective,
\begin{equation}
\mathcal L_S(\theta,\phi)=\sum_{i=1}^N S\!\left(P^Y_\theta(\cdot\mid x_i),y_i\right),
\end{equation}
where $S$ may be CRPS, Energy Score, Variogram Score, or another proper score~\citep{matheson1976,gneiting2007,gneiting2008,jordan2019}. Exact sample-level feasibility and analytic density tractability are separate properties.

\subsection{Training and inference interface}
Algorithm~1 in the main paper gives the complete generic loop. The same interface supports tractable transformed likelihoods or sample-based proper scores; the Structural Feasibility Map is applied during both training and inference, so feasibility does not depend on a separate post-training repair step.

\section{Structural Feasibility Map Taxonomy}
The taxonomy below is organized by structural representation rather than by optimization constraint form. Affine maps encode free coordinates directly; nonnegative/ordered maps compose a Nonnegative Structural Map with cumulative structure; normalized/compositional maps combine nonnegative components with normalization; explicit nonlinear-equality/manifold maps parameterize known manifolds or graphs; composed maps enforce intersections of structural requirements; and functional maps embed initial or boundary conditions in the output form. This axis is complementary to projection taxonomies based on linear equality, nonlinear equality, or convex inequality classes.

\begin{longtable}{>{\raggedright\arraybackslash}p{0.12\linewidth}
>{\raggedright\arraybackslash}p{0.17\linewidth}
>{\raggedright\arraybackslash}p{0.33\linewidth}
>{\raggedright\arraybackslash}p{0.16\linewidth}}
\caption{Representative Structural Feasibility Maps. Coverage, boundary reachability, and boundary probability are distinct properties.}\label{tab:taxonomy}\\
\toprule
\textbf{Constraint} & \textbf{Feasible set} & \textbf{Structural map} & \textbf{Qualification}\\
\midrule\endfirsthead
\toprule
\textbf{Constraint} & \textbf{Feasible set} & \textbf{Structural map} & \textbf{Qualification}\\
\midrule\endhead
Non-negativity & $\R_+^d$ & $h_i(z)=\psiplus(z_i)$ & $\psibm$: boundary mass; $\psibr$: reachable boundary; $\psiint$: interior only.\\
Bounded interval & $[a,b]^d$ & $a+(b-a)\sin^2(z_i)$ or $a+(b-a)\sigma(z_i)$ & $\sin^2$ reaches endpoints; sigmoid excludes them.\\
Weak ordering & $y_1\le\cdots\le y_d$ & $y_1=z_1$, $y_k=y_{k-1}+\psiplus(z_k)$ & Nonnegative increments; boundary behavior depends on $\psiplus$.\\
Ordering implementations & same & $\psibm(z)=\pospart{z}$ or $\pospart{z}^{p}$; $\psibr(z)=z^2$; $\psiint(z)=\softplus(z)$ & Positive-part family gives tie mass; square is boundary-reachable; softplus is interior.\\
Linear equality & $Ay=b$ & $h(z)=y_0+Nz$, $AN=0$, $Ay_0=b$ & Exact and surjective for a full null-space basis; latent dimension $d-m$.\\
Simplex & $\Delta^{d-1}$ & $a_i=\psiplus(z_i)$, $y_i=a_i/\sum_j a_j$ when active; feasible fallback if all inactive & Normalized nonnegative structural map; face mass depends on $\psiplus$.\\
Simplex implementations & $\Delta^{d-1}$ & normalized $\psibm$, normalized $\psibr$, or exponential-normalization $\psiint$ & Boundary-mass, boundary-reachable, and interior simplex laws, respectively.\\
Unit circle & $y_1^2+y_2^2=1$ & $h(z)=(\cos z,\sin z)$ & Surjective, periodic, non-injective; one structural coordinate.\\
Scale-shape FDC & $B\ge0,R\ge0,P\in\Delta^{K-1},\Delta=RP$ & $B=\psibm(z_B)$, $R=\psibm(z_R)$, $P=\mathrm{Norm}_{\psibm}(z_P)$, $\Delta=RP$ & Composed nonlinear equality; computational latent dimension may exceed intrinsic DOF.\\
Initial condition & $y(t_0)=y_0$ & $y(t)=y_0+(t-t_0)g_\phi(z,t)$ & Exact at $t_0$; path coverage depends on $g_\phi$.\\
Boundary condition & endpoint constraints & endpoint interpolant plus vanishing residual factor & Exact endpoints; interior expressiveness is model-dependent.\\
\bottomrule
\end{longtable}

\subsection{Primary weak-order map}
For
\begin{equation}
\Cset_{\mathrm{ord}}=\{y\in\R^d:y_1\le\cdots\le y_d\},
\end{equation}
a generic ordered Structural Feasibility Map uses nonnegative increments
\begin{align}
\delta_k&=\psiplus(z_k),\\
y_1&=z_1,\\
y_k&=y_{k-1}+\delta_k,\qquad k=2,\ldots,d.
\end{align}
Every realization is weakly ordered for any $\psiplus:\R\to\R_{\ge0}$. The reported primary experiment chooses the boundary-mass implementation $\psibm(z)=\pospart{z}=\max(0,z)$, using the same positive-part implementation specified in the main experiment. This cumulative-increment construction is distinct from monotone-network constraints on the input-output function itself, but it shares the broader feasible-by-design motivation of monotonic architectures~\citep{you2017,wehenkel2019,runje2023,igel2024}. Since the entire half-line $z_k\le0$ maps to zero increment under this implementation, a continuous latent law can induce an atom at an exact tie. Other members of the same boundary-mass class include $\psibm(z)=\pospart{z}^{p}$ for $p>0$ and thresholded positive parts $\pospart{z-\tau}$.

\subsection{Primary simplex map}\label{app:simplex-map}
A generic normalized Structural Feasibility Map first constructs $a_i=\psiplus(z_i)$ and $S(z)=\sum_j a_j$. The reported primary experiment again uses the boundary-mass positive-part implementation $\psibm(z)=\pospart{z}$. For $S(z)>0$ define
\begin{equation}
y_i=\frac{a_i}{S(z)}.
\end{equation}
The all-inactive event $S(z)=0$ has positive probability under common continuous latent laws. The implementation resolves it by assigning the sample to the simplex vertex indexed by $\arg\max_i z_i$. This preserves feasibility and assigns the all-inactive event to a vertex. The fallback is permutation-equivariant when the maximizing coordinate is unique, hence almost surely under an absolutely continuous latent law. A fixed measurable tie-breaking convention defines the remaining cases; pointwise permutation equivariance at ties is not asserted. The event is explicitly monitored; in the final $d=5$ primary experiment its mean rate is $3.81\times10^{-4}$. This reported frequency alone does not characterize fallback use throughout optimization or establish negligible gradient bias.

The map is surjective onto the closed simplex: for any $y\in\Delta^{d-1}$, the latent choice $z=y$ is a valid preimage. Its computational latent dimension is $d$, even though the simplex has $d-1$ intrinsic degrees of freedom. This closed-boundary emphasis differs from logistic-normal/softmax-style interior parameterizations common in compositional modeling~\citep{aitchison1980,aitchison1982}; constructive simplex distributions and sparse mappings provide additional context for representation-dependent boundary behavior~\citep{sethuraman1994,martins2016}.

\subsection{Explicit nonlinear-equality maps}
For the unit circle
\begin{equation}
\Cset_{\mathrm{circ}}=\{y\in\R^2:y_1^2+y_2^2=1\},
\end{equation}
the structural angle map
\begin{equation}
h_{\mathrm{circ}}(z)=(\cos z,\sin z)
\end{equation}
is surjective onto the circle and uses one structural coordinate. It is periodic and therefore non-injective; \SLPProbHard{} does not require a unique latent preimage for every feasible output. The matched radial reference maps ambient $u\in\R^2$ to $u/\|u\|_2$ with a deterministic feasible fallback for the exact-zero vector.

The real-world scale-shape construction is a composed map. For a seven-level weakly increasing FDC, let $K=6$, $B=Q_0$, $\Delta_i=Q_{i+1}-Q_i$, $R=\sum_i\Delta_i$, and $P\in\Delta^{K-1}$. The structural implementation uses
\begin{equation}
B=\pospart{z_B},\qquad R=\pospart{z_R},\qquad
P_i=\frac{\pospart{z_{P,i}}}{\sum_j\pospart{z_{P,j}}},
\qquad \Delta_i=RP_i,
\end{equation}
with the same feasible simplex fallback used by the normalized positive-part map. It therefore enforces nonnegative baseline/range, compositional shape, the bilinear equalities $\Delta_i=RP_i$, and nonnegative increments simultaneously. For $d=7$, the mathematical structural object has 7 intrinsic degrees of freedom but the implementation uses 8 computational latent coordinates because $P\in\Delta^5$ is represented with six normalized nonnegative coordinates.

\section{Detailed Structural-Map Remarks}
\subsection{Ordering and monotone paths}
The ordering and monotone-path constructions use the same cumulative Nonnegative Structural Map principle. The constraint only requires $\psiplus(z)\ge0$; probabilistic boundary behavior depends on which implementation is selected. Boundary-mass choices include the positive-part family $\pospart{z}^{p}$, $p>0$, and thresholded positive parts $\pospart{z-\tau}$, because a latent region of positive measure maps exactly to zero. The reported primary experiment uses $p=1$. Boundary-reachable choices such as $z^2$ can reach a tie but assign it zero probability under an absolutely continuous latent law. Interior choices such as softplus give strictly positive increments in real arithmetic. Floating-point underflow or rounding during cumulative addition can nevertheless produce exact numerical ties.

\subsection{Simplex maps}
The simplex Structural Feasibility Map can be constructed by normalizing nonnegative structural components, $a_i=\psiplus(z_i)$ and $y_i=a_i/\sum_j a_j$, together with a feasible rule for the all-inactive event when the chosen map permits it. The reported primary boundary-mass implementation uses $\psibm(z)=\pospart{z}$ and an argmax-vertex fallback. Normalized positive-part powers remain boundary-mass alternatives. Normalized squares are boundary-reachable and are surjective onto the closed simplex (aside from the singular all-zero latent vector handled numerically), but individual coordinate zeros require a measure-zero latent event under continuous Gaussian laws. Exponential normalization (softmax/logistic-normal) is a classical interior simplex representation~\citep{aitchison1980,aitchison1982}; sparsemax provides another mapping capable of exact zero coordinates~\citep{martins2016}. Thus ``simplex map'' is the structural family, while positive-part, square, softmax, and related transforms are implementations with different induced face probabilities.

\subsection{Initial and boundary conditions}
An initial condition $y(t_0)=y_0$ can be enforced by
\begin{equation}
y(t)=y_0+(t-t_0)g_\phi(z,t).
\end{equation}
Two endpoint conditions can be embedded in an interpolating ansatz plus a residual term multiplied by $(t-t_0)(t-t_1)$. These remain part of the general taxonomy but are not empirically benchmarked in this study.

\section{Complete Theory and Proofs}\label{app:theory}
\paragraph{Scope of this theoretical appendix.}
This appendix collects established identities and approximation results used to analyze \SFLP{}s. Rank-nullity, standard universal approximation, basic pushforward identities, truncated-Gaussian moments, and standard circular-distribution facts are used as foundations rather than claimed as standalone novelties. The \SLPProbHard{} contribution is their specialization and organization to characterize representation coverage, boundary probability, stochastic coordinates, covariance/expressiveness, and parameterization bias within a structural probabilistic predictive law.

\subsection{Pushforward and sample-level feasibility}
\begin{proposition}[Feasibility under pushforward]
Let $\Cset(x)$ be measurable, let $h_{\phi,x}:\Zset\rightarrow\Yset$ be measurable with $h_{\phi,x}(\Zset)\subseteq\Cset(x)$, and let $P^Z_\theta(\cdot\mid x)$ be any probability measure. If $P^Y_\theta=h_{\phi,x}\push P^Z_\theta$, then
\begin{equation}
P^Y_\theta(\Cset(x)\mid x)=1.
\end{equation}
\end{proposition}
\begin{proof}
For $B=\Yset\setminus\Cset(x)$, the image condition implies $h_{\phi,x}^{-1}(B)=\varnothing$. By the definition of pushforward, $P^Y_\theta(B\mid x)=0$.
\end{proof}

\begin{corollary}[Sample-level feasibility]
Every realization $z^{(s)}\in\Zset$ produces $h_\phi(x,z^{(s)})\in\Cset(x)$.
\end{corollary}
\begin{proof}
This is immediate from the image condition and holds independently for every sample.
\end{proof}

\subsection{Affine moment propagation}
\begin{proposition}[Affine moment propagation]
If $h(z)=Bz+c$, then
\begin{align}
\E[Y\mid x]&=B\E[Z\mid x]+c,\\
\Cov(Y\mid x)&=B\Cov(Z\mid x)B^\top.
\end{align}
\end{proposition}
The result provides exact first- and second-moment propagation for affine equalities and additive identities.

\subsection{Distributional propagation through structural maps}\label{app:dist-propagation}
Let
\begin{equation}
Z\mid x\sim P^Z_\theta(\cdot\mid x),\qquad Y=h(Z),
\end{equation}
with latent mean $\mu_Z=\E[Z\mid x]$ and covariance $\Sigma_Z=\Cov(Z\mid x)$. The feasible predictive distribution is exactly the pushforward $P^Y_\theta=h\push P^Z_\theta$. Whenever the moments exist,
\begin{align}
\mu_Y&=\E[h(Z)\mid x],\\
\Sigma_Y&=\E\!\left[(h(Z)-\mu_Y)(h(Z)-\mu_Y)^\top\mid x\right].
\end{align}
These identities are exact for every structural map; what changes across constraint families is whether they admit closed-form evaluation. For affine $h(z)=Bz+c$, the preceding proposition gives $\mu_Y=B\mu_Z+c$ and $\Sigma_Y=B\Sigma_ZB^\top$. For nonlinear maps, \SLPProbHard{} need not approximate the full predictive law as Gaussian: samples from $P^Z_\theta$ can be transformed through $h$ directly. This is especially important when a closed structural map induces discrete probability mass on feasible boundaries.

\paragraph{Positive-part boundary-mass implementation.}
Consider the primary boundary-mass implementation $R=\pospart{Z}=\max(0,Z)$ with
\begin{equation}
Z\sim\mathcal N(\mu,\sigma^2),\qquad \sigma>0,\qquad \alpha=\mu/\sigma,
\end{equation}
and let $\varphi$ and $\Phi$ denote the standard-normal density and distribution function. The pushforward has an atom at the feasible boundary,
\begin{equation}
\Pr(R=0)=\Phi(-\alpha),\label{eq:pospart-atom}
\end{equation}
and a continuous component on $(0,\infty)$. Its first two raw moments are
\begin{align}
m(\mu,\sigma)\equiv\E[R]
&=\sigma\varphi(\alpha)+\mu\Phi(\alpha),\label{eq:pospart-mean}\\
\E[R^2]
&=(\mu^2+\sigma^2)\Phi(\alpha)+\mu\sigma\varphi(\alpha),\label{eq:pospart-second}
\end{align}
so
\begin{equation}
v(\mu,\sigma)\equiv\Var(R)=\E[R^2]-m(\mu,\sigma)^2.\label{eq:pospart-var}
\end{equation}
These expressions follow by writing $Z=\mu+\sigma U$, $U\sim\mathcal N(0,1)$, and evaluating the first two Gaussian tail integrals over $U>-\alpha$. Thus the latent location and scale determine both ordinary predictive moments and the probability assigned exactly to the feasible boundary.

\paragraph{Weak ordering with independent Gaussian coordinates.}
For the primary ordering map,
\begin{align}
Y_1&=Z_1,\\
Y_k&=Z_1+\sum_{j=2}^{k}\pospart{Z_j},\qquad k=2,\ldots,d,
\end{align}
assume the primary diagonal-Gaussian latent model $Z_j\sim\mathcal N(\mu_j,\sigma_j^2)$ independently. Define $m_j=m(\mu_j,\sigma_j)$ and $v_j=v(\mu_j,\sigma_j)$ using Eqs.~\eqref{eq:pospart-mean}-\eqref{eq:pospart-var}. Then
\begin{align}
\boxed{\E[Y_k]}&=\mu_1+\sum_{j=2}^{k}m_j,\label{eq:order-mean}\\
\boxed{\Cov(Y_k,Y_\ell)}&=\sigma_1^2+\sum_{j=2}^{\min(k,\ell)}v_j.\label{eq:order-cov}
\end{align}
Moreover, the adjacent equality event has exact probability
\begin{equation}
\boxed{\Pr(Y_k=Y_{k-1})}=\Phi\!\left(-\frac{\mu_k}{\sigma_k}\right),\qquad k\ge2.\label{eq:order-tie}
\end{equation}
The moment identities follow from linearity and independence: two ordered outputs share $Z_1$ and exactly the increment terms up to $\min(k,\ell)$. Hence the network-predicted latent location and scale parameters directly determine the predictive probability of an exact tie.

\paragraph{Nonnegative ordering.}
For the composed map
\begin{equation}
Y_k=\sum_{j=1}^{k}\pospart{Z_j},\qquad k=1,\ldots,d,
\end{equation}
the same independent diagonal-Gaussian assumption gives
\begin{align}
\boxed{\E[Y_k]}&=\sum_{j=1}^{k}m_j,\\
\boxed{\Cov(Y_k,Y_\ell)}&=\sum_{j=1}^{\min(k,\ell)}v_j.
\end{align}
The adjacent-tie probability remains
\begin{equation}
\boxed{\Pr(Y_k=Y_{k-1})}=\Phi\!\left(-\frac{\mu_k}{\sigma_k}\right),\qquad k\ge2,
\end{equation}
while independence yields the exact zero-state probability
\begin{equation}
\boxed{\Pr(Y_k=0)}=\prod_{j=1}^{k}\Phi\!\left(-\frac{\mu_j}{\sigma_j}\right).\label{eq:nonneg-order-zero}
\end{equation}
Thus the same structural coordinates control both ordering ties and exact zero-valued states.

\paragraph{Boundary-reachable square map.}
For $R=Z^2$ with $Z\sim\mathcal N(\mu,\sigma^2)$,
\begin{align}
\E[R]&=\mu^2+\sigma^2,\\
\Var(R)&=2\sigma^4+4\mu^2\sigma^2.
\end{align}
Nevertheless, $\Pr(R=0)=0$ under a continuous Gaussian latent law. Boundary reachability and positive boundary probability therefore remain distinct even when output moments are analytically available.

\paragraph{Normalized boundary-mass simplex map.}
For
\begin{equation}
A_i=\pospart{Z_i},\qquad Y_i=\frac{A_i}{\sum_jA_j}
\end{equation}
whenever at least one coordinate is active, define the random active set $\mathcal A=\{i:Z_i>0\}$. Under conditionally independent Gaussian coordinates with $\sigma_i(x)>0$, for any fixed nonempty index set $S$,
\begin{equation}
\Pr(\mathcal A=S\mid x)=\prod_{i\in S}\Phi\!\left(\frac{\mu_i}{\sigma_i}\right)
\prod_{j\notin S}\Phi\!\left(-\frac{\mu_j}{\sigma_j}\right),\qquad S\neq\varnothing.\label{eq:simplex-active}
\end{equation}
Here and below the dependence of $\mu_i,\sigma_i$ on $x$ is suppressed. Conditional on $\mathcal A=S\ne\varnothing$, the predictive sample lies in the relative interior of the simplex face indexed by $S$. These are latent active-set probabilities. For a singleton $S=\{i\}$, the output vertex probability additionally includes mass routed to $e_i$ from the all-inactive event:
\[
\Pr(Y=e_i\mid x)=\Pr(\mathcal A=\{i\}\mid x)
+\Pr(\mathcal A=\varnothing,\ I(Z)=i\mid x),
\]
where $I(Z)$ is the fallback's maximizing index, including its tie-breaking convention. The all-inactive event has probability
\begin{equation}
p_{\varnothing}(x)=\prod_{j=1}^{d}\Phi\!\left(-\frac{\mu_j}{\sigma_j}\right),
\end{equation}
and is handled by the feasible argmax-vertex fallback in Appendix~\ref{app:simplex-map}. Thus the structural pushforward is a mixture over simplex faces, with the additional fallback contribution included at the vertices. The mean can be decomposed as
\begin{equation}
\E[Y\mid x]=\sum_{S\neq\varnothing}\Pr(\mathcal A=S\mid x)\E[Y\mid\mathcal A=S,x]
+p_{\varnothing}(x)\E[Y\mid\mathcal A=\varnothing,x],
\end{equation}
with covariance obtained analogously by the law of total covariance. The conditional ratio moments generally do not reduce to simple elementary closed forms and are evaluated by Monte Carlo or numerical integration.

\paragraph{Circular structural map.}
If $Z\sim\mathcal N(\mu,\sigma^2)$ and $Y=(\cos Z,\sin Z)$, then the first two moments are exact:
\begin{align}
\E[\cos Z]&=e^{-\sigma^2/2}\cos\mu, &
\E[\sin Z]&=e^{-\sigma^2/2}\sin\mu,\\
\E[\cos^2 Z]&=\tfrac12\!\left(1+e^{-2\sigma^2}\cos 2\mu\right), &
\E[\sin^2 Z]&=\tfrac12\!\left(1-e^{-2\sigma^2}\cos 2\mu\right),\\
\E[\sin Z\cos Z]&=\tfrac12 e^{-2\sigma^2}\sin 2\mu.
\end{align}
Thus the circle benchmark supplies a nonlinear pushforward whose moments remain analytically tractable even though the induced Cartesian law is non-Gaussian.

\paragraph{Smooth nonlinear maps.}
For a differentiable structural map $h$ with latent mean $\mu_Z$ and covariance $\Sigma_Z$, the first-order delta approximation~\citep{oehlert1992} is
\begin{align}
\E[Y]&\approx h(\mu_Z),\\
\Cov(Y)&\approx J_h(\mu_Z)\Sigma_ZJ_h(\mu_Z)^\top.
\end{align}
This approximation is optional rather than required for feasibility. Exact pushforward sampling remains available whenever moments or densities are analytically inconvenient. In particular, a Gaussian moment approximation should not replace the full pushforward law when boundary atoms or face-supported probability mass are scientifically relevant.

\begin{table}[H]
\centering
\scriptsize
\begin{tabularx}{\linewidth}{p{0.18\linewidth}YYY}
\toprule
Structural map & Predictive-law structure & Moment propagation & Boundary characterization\\
\midrule
Affine equality & affine pushforward & exact mean/covariance & feasible affine manifold\\
Positive-part boundary-mass & atom at zero + continuous positive component & exact for Gaussian latent & $\Pr(Y=0)=\Phi(-\mu/\sigma)$\\
Weak ordering & cumulative mixed law & exact for independent Gaussian coordinates & exact adjacent-tie probabilities\\
Nonnegative ordering & cumulative mixed law & exact for independent Gaussian coordinates & exact zero and tie probabilities\\
Square & squared-Gaussian transform & exact Gaussian moments & boundary reachable, zero mass\\
Boundary-mass normalized simplex & mixture over simplex faces & conditional/numerical moments & analytic active-set probabilities\\
Smooth generic map & nonlinear pushforward & delta approximation or Monte Carlo & map-dependent\\
\bottomrule
\end{tabularx}
\caption{Distributional propagation under representative structural maps. Analytic tractability is separate from hard feasibility: every listed map is defined through its pushforward law, while closed-form summaries are used when available.}
\label{tab:dist-propagation}
\end{table}

\subsection{Samplewise differentiation and expected-score gradients}\label{app:pathwise}

\begin{proposition}[Samplewise chain rule]
Fix $x$ and $\theta_0$, and let $h$ be a fixed structural map. Assume that $\theta\mapsto z_\theta(x,\epsilon)$ is differentiable at
$\theta_0$ almost surely and that $h$ is differentiable at $z_{\theta_0}(x,\epsilon)$ almost surely.
Then $\theta\mapsto h(z_\theta(x,\epsilon))$ is differentiable at $\theta_0$ almost surely, with
\[
D_\theta h(z_\theta(x,\epsilon))\big|_{\theta_0}
=
J_h(z_{\theta_0}(x,\epsilon))
D_\theta z_\theta(x,\epsilon)\big|_{\theta_0}.
\]
\end{proposition}
\begin{proof}
Apply the ordinary chain rule on the probability-one event where both
derivatives exist.
\end{proof}

This samplewise statement does not by itself justify differentiating
an expected score. Let
$L(\theta)=\mathbb E_\epsilon[\ell(\theta,\epsilon)]$, where $\ell$
may include a complete predictive ensemble and its sample-based score.
A sufficient condition for differentiation at $\theta_0$ is that
$\ell(\theta_0,\epsilon)$ is integrable, that
$\ell(\cdot,\epsilon)$ is differentiable at $\theta_0$ almost surely,
and that it is locally Lipschitz in a neighborhood of $\theta_0$
with an integrable random Lipschitz bound. Dominated convergence
applied to difference quotients then gives
\[
\nabla_\theta L(\theta_0)
=
\mathbb E_\epsilon[
\nabla_\theta\ell(\theta_0,\epsilon)].
\]
The gradient of an expected finite-ensemble score is a separate
question from whether that score estimator is unbiased for a population scoring rule. The positive-part map is continuous and globally Lipschitz, despite its kink at zero. The normalized positive-part simplex map, however, contains an argmax-vertex fallback on the all-inactive region and is
discontinuous at some boundaries between fallback choices. Almost-everywhere samplewise differentiation therefore does not automatically establish unbiased expected-score gradients for the complete map. The corresponding estimator requires separate
justification or an explicitly reported approximation. Reparameterization methods for the latent sampler~\citep{figurnov2018,jankowiak2018} do not by themselves resolve discontinuities in the structural map. A small reported fallback frequency does not establish negligible gradient bias.

\subsection{Affine rank-nullity}
\begin{proposition}[Free-coordinate dimension]
Let $A\in\R^{m\times d}$ have rank $m<d$ and assume $Ay=b$ is consistent. Then $\Cset=\{y:Ay=b\}$ has dimension $q=d-m$. If $N\in\R^{d\times q}$ has linearly independent columns spanning $\ker(A)$, every feasible point has a unique representation $y=y_0+Nz$.
\end{proposition}
\begin{proof}
The rank-nullity theorem gives $\dim\ker(A)=d-m$. For any feasible $y$, $A(y-y_0)=0$, so $y-y_0\in\ker(A)$ and equals $Nz$; linear independence makes $z$ unique.
\end{proof}

\subsection{Constrained universal approximation}
\begin{theorem}[Approximation consequence under a continuous structural lifting]\label{thm:structural-ua}
Let $\Xset\subset\R^p$ be compact and $h:\Xset\times\R^q\rightarrow\R^d$ continuous with $h(x,z)\in\Cset(x)$. Let $F^\star:\Xset\rightarrow\R^d$ be continuous and feasible. Assume there exists a continuous lifting $z^\star:\Xset\rightarrow\R^q$ with $F^\star(x)=h(x,z^\star(x))$, and that the latent network family is a universal approximator of continuous maps into $\R^q$. Then $h(x,f_\theta(x))$ can uniformly approximate $F^\star$ while remaining feasible for all $x$.
\end{theorem}
If $h$ is uniformly $L_h$-Lipschitz in $z$,
\begin{equation}
\sup_x\|h(x,f_\theta(x))-F^\star(x)\|\le L_h\sup_x\|f_\theta(x)-z^\star(x)\|.
\end{equation}
The unconstrained approximation ingredient follows from standard universal-approximation results~\citep{cybenko1989,hornik1991,leshno1993}. This theorem is a direct structural consequence of standard universal approximation once a continuous feasible lifting exists; its role here is to characterize when a chosen representation preserves approximation capacity. The substantive assumption is therefore the existence of the lifting, a property of the chosen representation.

\begin{corollary}[Approximation for weak-order and simplex targets]\label{cor:explicit-liftings}
Let $X$ be compact and let the latent network family uniformly
approximate continuous maps into the required Euclidean latent space.
Every continuous weakly ordered target can be uniformly approximated
through the primary weak-order map. Every continuous simplex-valued
target can also be uniformly approximated through the normalized
positive-part map with its feasible fallback.
\end{corollary}

\begin{proof}
For a weakly ordered vector $y$, the continuous lifting
\[
s_{\mathrm{ord}}(y)
=
(y_1,y_2-y_1,\ldots,y_d-y_{d-1})
\]
satisfies $h_{\mathrm{ord}}(s_{\mathrm{ord}}(y))=y$.
The weak-order map is continuous, so Theorem~\ref{thm:structural-ua} applies.

For a simplex-valued target $F^\star$, use the continuous lifting
$z^\star(x)=F^\star(x)$. Choose a latent approximator satisfying
\[
\sup_{x\in X}
\|f_\theta(x)-F^\star(x)\|_\infty
<\delta<\frac{1}{2d}.
\]
Write $a(x)=[f_\theta(x)]_+$ and $S(x)=\sum_i a_i(x)$.
Since $F^\star_i(x)\ge0$ and $\sum_iF^\star_i(x)=1$,
\[
S(x)\ge1-d\delta>\frac12.
\]
The fallback is therefore never used. Moreover,
\begin{align*}
\|h_\Delta(f_\theta(x))-F^\star(x)\|_1
&\le
\|a(x)/S(x)-a(x)\|_1
+\|a(x)-F^\star(x)\|_1\\
&=
|1-S(x)|+\|a(x)-F^\star(x)\|_1\\
&\le
2\|a(x)-F^\star(x)\|_1\\
&\le 2d\delta.
\end{align*}
Taking $\delta$ arbitrarily small proves uniform approximation.
This local argument does not require global continuity of the
fallback map.
\end{proof}

\begin{remark}[Circle topology]
The circle map $h_{\mathrm{circ}}(z)=(\cos z,\sin z)$ is continuous and surjective, but it admits no global continuous section $S^1\to\R$; this is the standard covering-space/topological obstruction associated with the angle map~\citep{hatcher2002}: a globally continuous choice of angle would contradict the topology of the circle. A measurable branch-cut section exists, so surjectivity still supports the measurable distributional-approximation result below, while Theorem~\ref{thm:structural-ua} does not apply globally through a continuous angle lifting. This separates map surjectivity from existence of a continuous inverse/section.
\end{remark}

\subsection{Distributional approximation and parameterization bias}
\begin{theorem}[Classical pushforward approximation consequence under a measurable section]
Let $h:\Zset\to\Cset$ be measurable and surjective between standard Borel spaces, and suppose a measurable section $s:\Cset\to\Zset$ exists with $h(s(y))=y$. If a latent family is dense in probability measures on $\Zset$ under total variation, then its pushforward family is dense in probability measures on $\Cset$ under total variation.
\end{theorem}
\begin{proof}
Let $\nu$ be a probability measure on $\mathcal C$ and define
$P=s_\#\nu$. Since $h\circ s$ is the identity on $\mathcal C$,
$h_\#P=\nu$. By the assumed total-variation density, choose latent
laws $P_n$ with $\|P_n-P\|_{\mathrm{TV}}\to0$. Then
\[
\|h_\#P_n-\nu\|_{\mathrm{TV}}
=
\sup_{B}
|P_n(h^{-1}(B))-P(h^{-1}(B))|
\le
\|P_n-P\|_{\mathrm{TV}}
\to0,
\]
where the supremum is over measurable subsets of $\mathcal C$.
\end{proof}

This result is conditional on the stated density assumption and is
not a universal-expressiveness guarantee for the Gaussian heads
used in the experiments. In particular, absolutely continuous
latent families cannot approximate every atomic probability measure
on Euclidean latent space in total variation.

This standard pushforward consequence is included to characterize distributional expressiveness of a chosen representation, not as a new measure-theoretic result.

\begin{proposition}[Wasserstein stability]
If $h$ is $L_h$-Lipschitz, then $W_p(h\push P,h\push Q)\le L_hW_p(P,Q)$ whenever the moments exist, using the standard Wasserstein/optimal-transport metric framework~\citep{villani2009}. More generally, transport-map formulations provide a useful probability-theoretic lens for deterministic pushforwards of latent laws~\citep{marzouk2016}.
\end{proposition}

\begin{proposition}[Parameterization bias]
Let $h:\mathcal Z\to\mathcal C$ be measurable, and assume
$H=h(\mathcal Z)$ is measurable. Every pushforward law
$P^Y=h_\#P^Z$ satisfies $P^Y(H)=1$. Consequently, any target law
$\nu$ with $\nu(\mathcal C\setminus H)>0$ cannot be represented
exactly by this map.
\end{proposition}
\begin{proof}
Because $h(z)\in H$ for every $z\in\mathcal Z$,
$h^{-1}(H)=\mathcal Z$, and hence $P^Y(H)=1$.
\end{proof}

Reachability, topological support, and event probability are distinct.
A point outside $h(\mathcal Z)$ may belong to the topological support
as a limit point. For example, softplus never outputs zero in real
arithmetic, but zero belongs to the topological support of its
pushforward under a nondegenerate Gaussian latent law.

\subsection{Classical identity used as a reachability and boundary-mass taxonomy}
\begin{proposition}[Boundary-preimage identity]\label{prop:boundary}
Let $B\subseteq\Cset$ be measurable and let $P^Y=h\push P^Z$. Then
\begin{equation}
P^Y(B)=P^Z\!\left(h^{-1}(B)\right).
\end{equation}
Consequently: (i) if $h^{-1}(B)=\varnothing$, the boundary event is excluded; (ii) if $h^{-1}(B)$ is nonempty but has zero latent probability, the boundary is reachable but receives no predictive mass; and (iii) if the preimage has positive latent probability, the predictive law has positive boundary mass.
\end{proposition}
\begin{proof}
This is the defining identity of the pushforward measure applied to the measurable set $B$.
\end{proof}

The following image and boundary-mass statements concern real arithmetic. Floating-point underflow and rounding can change exact numerical zero or tie diagnostics.

This elementary identity becomes a useful design taxonomy when applied to constrained probabilistic maps. For an ordering increment $\Delta$ and boundary event $B=\{\Delta=0\}$:
\begin{itemize}
\item $\Delta=\softplus(Z)$ has empty finite preimage and excludes the boundary;
\item $\Delta=Z^2$ has preimage $\{0\}$, which has probability zero under an absolutely continuous latent law;
\item $\Delta=\pospart{Z}$ has preimage $(-\infty,0]$, which has positive probability whenever the latent law assigns positive mass there. For $Z\sim\mathcal N(\mu,\sigma^2)$, this mass is $\Phi(-\mu/\sigma)$.
\end{itemize}
For the normalized positive-part boundary-mass simplex, the face event $\{Y_i=0\}$ includes latent states with $Z_i\le0$ and at least one other active coordinate, and therefore can have positive probability under a continuous latent law. By contrast, normalized squares require the measure-zero event $Z_i=0$ for a coordinate to be exactly zero, and softmax excludes zero entirely. Feasibility, boundary reachability, and boundary probability are therefore distinct modeling properties.

\subsection{Finite precision}
All feasibility theorems above are real-arithmetic statements. Matrix products, summation, null-space reconstruction, and projection can introduce small residuals in float32; this is the standard distinction between exact real-arithmetic identities and their floating-point evaluation~\citep{goldberg1991,higham2002}. A numerical threshold is therefore an evaluation convention, not a relaxation of the mathematical constraint. Continuous residuals are the primary numerical evidence; thresholded violation rates are secondary diagnostics.

\section{Extended Comparison with Projection-Based Methods}\label{app:comparison}
\subsection{Mechanism-centered comparison}
\begin{table}[H]
\centering\footnotesize
\resizebox{0.94\linewidth}{!}{%
\begin{tabular}{p{0.18\linewidth}p{0.20\linewidth}p{0.18\linewidth}p{0.16\linewidth}p{0.20\linewidth}}
\toprule
\textbf{Method} & \textbf{Primary mechanism} & \textbf{Probabilistic predictive law} & \textbf{Boundary behavior} & \textbf{Main qualification}\\
\midrule
DC3~\citep{donti2021} & Equality completion plus iterative inequality correction & Not a dedicated probabilistic framework & Completion or correction dependent & Requires correction machinery.\\
\citet{frerix2020} & Direct parameterization for homogeneous linear inequalities & Deterministic focus & Represented feasible boundary available & Specialized constraint class.\\
ProbHard\-E2E~\citep{utkarsh2025} & Differentiable probabilistic projection & Yes & Determined by projected law & Broad supported constraint classes; projection machinery required.\\
Soft-Radial~\citep{schneider2026} & Radial map into convex-set interior & Not central & Interior-oriented & Requires convex geometry and an interior anchor.\\
\SLPProbHard{} & Stochastic latent law plus structural feasibility map & Yes & Map-dependent: excluded, reachable-only, or positive mass & Requires a suitable structural map.\\
\bottomrule
\end{tabular}}
\caption{Representative mechanisms. Several approaches provide hard feasibility; \SLPProbHard{} is distinguished by making a constraint-specific structural map part of a probabilistic pushforward model.}
\end{table}

\subsection{Detailed comparison with ProbHardE2E}
ProbHardE2E and \SLPProbHard{} target complementary regimes. Projection is preferable when a feasible set is implicit, complex, or easier to project onto than to parameterize. \SLPProbHard{} is attractive when an explicit generative representation is available and scientifically acceptable. The compatible affine experiment below gives a direct empirical comparison using the frozen official source and a verified generic $A,b$ adapter; it is intentionally narrower than ProbHardE2E's full supported constraint scope.

\begin{longtable}{p{0.17\linewidth}p{0.20\linewidth}p{0.20\linewidth}p{0.17\linewidth}}
\caption{Mechanistic comparison with probabilistic projection.}\label{tab:comparison}\\
\toprule
\textbf{Criterion} & \textbf{Projection / ProbHardE2E family} & \textbf{\SLPProbHard{}} & \textbf{Implication}\\
\midrule\endfirsthead
\toprule
\textbf{Criterion} & \textbf{Projection / ProbHardE2E family} & \textbf{\SLPProbHard{}} & \textbf{Implication}\\
\midrule\endhead
Core mechanism & Unconstrained probabilistic representation followed by projection/correction & $Y=h(Z)$ with $h(\Zset)\subseteq\Cset$ & Different constructions of the feasible law.\\
Organizing taxonomy & Linear/nonlinear equalities and convex inequalities under projection & Affine, nonnegative/ordered, normalized/compositional, composed, and functional structural families & Different axes of generality.\\
Constraint generality & Broad for supported equality and convex-inequality classes & Selective and map-dependent & Projection is generally broader.\\
Linear equality coordinates & Ambient prediction before projection & $q=d-m$ free coordinates & Exact intrinsic reduction in this family.\\
Nonlinear coordinates & Ambient or method-specific & May be intrinsic or redundant & Dimensional benefit is not universal.\\
Projection metric & Selected geometry may matter & No projection metric for explicit maps & Removes metric choice.\\
Solver / correction stage & Required by the chosen projection/correction mechanism & None when map is explicit & Strongest for maps that replace nontrivial solves.\\
Boundary probability & Induced by projection geometry and base law & Induced by structural-map preimages and latent law & Both can create boundary mass; mechanisms differ.\\
Affine moments & Can be exact & Exact & Tie in tractability for affine maps.\\
Main risk & Projection cost/conditioning or solver complexity & Parameterization bias and map-induced optimization effects & Complementary failure modes.\\
\bottomrule
\end{longtable}

\section{Experimental Protocol and Reproducibility}\label{app:protocol}
\subsection{Frozen benchmark families and held-out seeds}
The reported evidence uses development-independent seed ranges after map and soft-penalty choices were frozen:
\begin{itemize}
\item linear primary: seeds 20-29; linear scaling: 30-34; covariance ablation: 40-44;
\item ordering primary: seeds 60-69; ordering scaling: 70-74;
\item real-world weak-order Benchmark A and composed nonnegative-ordering Benchmark B: seeds 80-89, intentionally matched across A/B;
\item simplex primary: seeds 100-109; simplex scaling: 110-114;
\item synthetic composed nonnegative-ordering Benchmark B: seeds 115-119;
\item nonlinear circle primary: seeds 130-139;
\item real-world affine FDC coherence: seeds 150-159;
\item real-world nonlinear scale-shape coherence: seeds 170-179.
\end{itemize}
All reported training experiments use float32 on CPU. Primary probabilistic evaluations use 100 predictive samples. Smoke/development seeds are disjoint from the reported primary seeds.

\begin{table}[H]
\centering\small
\begin{tabular}{p{0.36\linewidth}p{0.54\linewidth}}
\toprule
Frozen item & Value used in the primary synthetic benchmarks\\
\midrule
Train / validation / test size & 3000 / 750 / 750\\
Backbone & MLP, hidden width 128, depth 3, dropout 0\\
Batch size & 128\\
Maximum epochs / early stopping & 60 / patience 20, minimum improvement $10^{-6}$\\
Learning rate / weight decay & $10^{-3}$ / $10^{-5}$\\
Gradient clipping & 5.0 where exposed by the frozen runner\\
Monte Carlo samples & 12 training, 32 validation, 100 evaluation\\
Probabilistic objective & 0.5 marginal CRPS + 0.5 Energy Score\\
Precision / device & float32 / CPU\\
Primary network dimensions & input 12; outputs 11 (linear), 5 (ordering/simplex), 2 (circle)\\
\bottomrule
\end{tabular}
\caption{Primary synthetic training configuration. Competing methods share the same backbone/training protocol within each benchmark; only the constraint mechanism and corresponding output/latent head differ where required.}
\label{tab:repro-hparams}
\end{table}

\subsection{Synthetic-DGP transparency and fairness}
The synthetic tasks are diagnostic benchmarks rather than claims of universal DGP dominance. Table~\ref{tab:dgp-fairness} makes structural alignment explicit. Alignment can favor an \SLPProbHard{} representation, so the scientific question for each task is stated separately from raw score comparison.

\begin{table}[H]
\centering\scriptsize
\resizebox{\linewidth}{!}{%
\begin{tabular}{p{0.12\linewidth}p{0.31\linewidth}p{0.18\linewidth}p{0.16\linewidth}p{0.24\linewidth}}
\toprule
Benchmark & Frozen target-generation structure & Same structure as primary \SLPProbHard{} map? & Potential \SLPProbHard{} alignment & Intended scientific question\\
\midrule
Affine & Feasible latent regression is mapped through a fixed affine null-space representation $y=y_0+Nz$. & Yes & Yes: coordinates match the feasible DGP & Can a probabilistic model learn only the $q=d-m$ free directions while matching ambient projection?\\
Weak ordering & A base response is combined with cumulative nonnegative positive-part increments, producing exact-tie mass. & Yes & Yes: boundary mechanism is aligned & Does a boundary-mass ordering representation change tie probability and proper scores relative to isotonic projection?\\
Simplex & Nonnegative components with an exact-zero mechanism are normalized to the simplex, producing face-supported observations. & Yes, at the structural level & Yes: boundary-capable normalization is aligned & How do direct boundary-capable normalization and Euclidean simplex projection differ in face probability and predictive scoring?\\
Circle & A conditional angle is drawn from a von Mises law and mapped as $(\cos\Theta,\sin\Theta)$. Learned \SLPProbHard{} uses a Gaussian angle; projection uses an ambient bivariate Gaussian. & Geometry yes; latent family no & Partial & Does one structural angle provide competitive hard-feasible prediction relative to ambient radial projection under a non-Gaussian angular DGP?\\
Composed nonneg. order & Targets are generated directly by cumulative positive-part from the first coordinate onward. & Yes, exactly & Yes, deliberately strong & Mechanism validation: do the predicted zero/tie boundary effects follow the preimage analysis when the structural DGP is known?\\
\bottomrule
\end{tabular}}
\caption{Synthetic-DGP transparency. The table separates structural alignment from the question each benchmark is designed to test. The composed benchmark is deliberately favorable and is interpreted only as mechanism validation.}
\label{tab:dgp-fairness}
\end{table}

\subsection{Baselines}
Linear equality compares unconstrained prediction, validation-selected soft penalty, post-hoc orthogonal projection, end-to-end orthogonal projection (Projection-Or), covariance-weighted/oblique projection, and \SLPProbHard{}. Ordering compares unconstrained output, soft crossing penalty, post-hoc sorting, post-hoc isotonic correction, end-to-end isotonic projection, and positive-part SLP. Simplex compares unconstrained output, soft constraint loss, post-hoc simplex projection, end-to-end Euclidean simplex projection, and normalized positive-part SLP. Circle compares unconstrained bivariate Gaussian output, soft radius penalty, post-hoc radial projection, end-to-end radial projection, and one-coordinate circular SLP. The real-world affine extension uses unconstrained/soft augmented outputs, post-hoc affine projection, end-to-end affine projection, and structural affine SLP. The real-world nonlinear extension uses unconstrained/soft $(B,R,P,\Delta)$ outputs, post-hoc alternating feasible correction, end-to-end alternating feasible correction, and structural scale-shape SLP.

\subsection{Shared hydrological FDC source and preprocessing}
The hydrological experiments reuse one frozen dataset created before the affine and nonlinear extensions. Daily river discharge comes from the \emph{GloFAS v5.0 hydrological reanalysis} of the European Commission Joint Research Centre~\citep{grimaldi2026glofas}; GloFAS uses the LISFLOOD hydrological modelling system~\citep{vanderknijff2010,burek2013}. Seven Emilia-Romagna basins are used: Crostolo, Enza, Panaro, Parma, Secchia, Taro, and Trebbia. Basin-aggregated precipitation, minimum-temperature, and maximum-temperature descriptors are derived from ERACLITO61~\citep{antolini2016}, together with season indicators, basin area, and basin identity. The common 1996-2022 period yields 756 seasonal basin-season-year samples. The temporal split is 532 training samples (1996-2014), 112 validation samples (2015-2018), and 112 held-out test samples (2019-2022), with 22 standardized predictor features.

Using the conventional exceedance-probability interpretation of flow-duration curves~\citep{burgan2018,vogel1994,vogel1995} and the multi-output FDC precedent of \citet{worland2019}, the seven FDC components are
\begin{equation}
Q=[Q_{95},Q_{90},Q_{75},Q_{50},Q_{25},Q_{10},Q_{05}],
\end{equation}
where $Q_{95}$ is the ordinary 5th percentile and $Q_{05}$ the ordinary 95th percentile of the seasonal daily discharge distribution. Thus physically admissible FDCs satisfy $0\le Q_{95}\le\cdots\le Q_{05}$. The original ordering experiment uses $\log(1+Q)$ targets and transforms samples back by $\exp(y)-1$ without clipping.

The affine and nonlinear extensions do \emph{not} rerun the daily-discharge, predictor-NetCDF, basin-mask, FDC-quantile, or split construction. They read the frozen raw $Q$ arrays and the already-standardized $x$ arrays. A single positive multiplicative scale,
\begin{equation}
s_Q=\sqrt{\operatorname{mean}(Q_{\mathrm{train}}^2)}=43.7190796,
\end{equation}
is fitted on training raw FDC values only. No centering is used. The scaled target $Q^\star=Q/s_Q$ preserves both adjacent-difference and scale-shape equalities exactly. The frozen source SHA-256 is \texttt{18cd85cac0405b740866c088a61566c93e4ec6d020bbf881a9f33ee154afff47}.

\subsection{Real-world nonnegative ordering}
The composed map
\begin{equation}
Q_k=\sum_{j=1}^{k}\pospart{z_j}
\end{equation}
guarantees $0\le Q_1\le\cdots\le Q_d$. The matched hard reference is the frozen nonnegative isotonic projection. This representation directly tests physical order/non-negativity and boundary mass in the FDC itself.

\subsection{Real-world affine FDC coherence}\label{app:fdc-affine}
For scaled $Q^\star\in\R^7$, define six adjacent increments
\begin{equation}
\Delta_i=Q^\star_{i+1}-Q^\star_i,\qquad i=0,\ldots,5,
\end{equation}
and augment the target as
\begin{equation}
Y_{\mathrm{aff}}=(Q^\star_0,\ldots,Q^\star_6,\Delta_0,\ldots,\Delta_5)\in\R^{13}.
\end{equation}
The six affine constraints are
\begin{equation}
r_i(Y)=Q^\star_i-Q^\star_{i+1}+\Delta_i=0.
\end{equation}
If $A\in\R^{6\times13}$ collects these rows, $\operatorname{rank}(A)=6$, so the affine feasible set has intrinsic dimension $13-6=7$. The structural map takes seven stochastic coordinates,
\begin{equation}
h_{\mathrm{aff}}(z)=\left(z,\operatorname{diff}(z)\right),
\end{equation}
and therefore satisfies every equality algebraically. The end-to-end affine reference predicts a 13-dimensional diagonal Gaussian and applies the exact Euclidean projection onto $AY=0$. Hence the hard comparison is 13 ambient stochastic coordinates versus 7 structural coordinates, a 46.1538\% reduction.

This benchmark intentionally isolates \emph{affine coherence only}. Neither the affine projection nor $h_{\mathrm{aff}}$ additionally enforces non-negativity or FDC ordering; negative-value and crossing rates are therefore reported only as secondary diagnostics and must not be described as violations of the affine constraint. The validation-selected soft weight is 0.01; primary seeds are 150-159. Training uses the same 128-width/depth-3 MLP, batch 128, 60-epoch cap, $10^{-3}$ learning rate, $10^{-5}$ weight decay, 12/32 Monte Carlo samples for train/validation, and 100 test samples.

Primary predictive reporting is in reconstructed raw-$Q$ space, $Q=s_Q Q^\star$. Augmented $(Q^\star,\Delta^\star)$ scores are secondary. Constraint metrics use $r_i$ and report mean absolute residual, mean squared residual, maximum absolute residual, and vector violation rate $\mathbf{1}[\max_i|r_i|>10^{-5}]$.

\subsection{Real-world nonlinear scale-shape coherence}\label{app:fdc-nonlinear}
For the same scaled FDC, define
\begin{equation}
B=Q^\star_0,\qquad
\Delta_i=Q^\star_{i+1}-Q^\star_i,\qquad
R=\sum_{i=0}^{5}\Delta_i=Q^\star_6-Q^\star_0.
\end{equation}
For $R>0$, define the compositional shape
\begin{equation}
P_i=\frac{\Delta_i}{R},\qquad P_i\ge0,\qquad \sum_iP_i=1.
\end{equation}
The augmented target is
\begin{equation}
Y_{\mathrm{nl}}=(B,R,P_0,\ldots,P_5,\Delta_0,\ldots,\Delta_5)\in\R^{14},
\end{equation}
with nonlinear bilinear equalities
\begin{equation}
\Delta_i=RP_i,\qquad i=0,\ldots,5.
\end{equation}
No train/validation/test FDC is flat at the frozen scaled threshold $R\le10^{-8}$ (flat counts 0/0/0). The protocol nevertheless freezes the convention $P=e_1$ for an exactly flat FDC, under which $\Delta=RP=0$ remains feasible.

The structural map uses
\begin{align}
B&=\pospart{z_B}, &
R&=\pospart{z_R},\\
P&=\operatorname{NormPosPart}(z_P), &
\Delta&=RP.
\end{align}
It therefore enforces $B\ge0$, $R\ge0$, $P\ge0$, $\sum_iP_i=1$, $\Delta=RP$, and consequently nonnegative increments and a nonnegative weakly ordered reconstructed FDC. For $d=7$, the mathematical intrinsic DOF is 7. The boundary-capable normalized positive-part shape uses six computational coordinates for the five-dimensional simplex, so \SLPProbHard{} uses $1+1+6=8$ computational stochastic coordinates. The ambient correction model predicts all 14 coordinates; the reduction is 42.8571\%. This is reported as a \emph{computational} stochastic-coordinate reduction, not as an 8-dimensional intrinsic manifold claim.

The end-to-end hard comparator is a five-iteration differentiable alternating feasible correction. Given ambient predictions $(B_0,R_0,P_0,\Delta_0)$, it projects $B_0$ to $\R_+$ and alternates
\begin{align}
R &\leftarrow \pospart{\frac{R_0+P^\top\Delta_0}{1+\|P\|_2^2}},\\
P &\leftarrow \Pi_{\Delta}\!\left(\frac{P_0+R\Delta_0}{1+R^2}\right),
\end{align}
then sets $\Delta=RP$. These are exact conditional minimizers for the two blocks of a feasible-correction objective, but the finite alternating procedure is \emph{not claimed to be the exact joint Euclidean projection}. The validation-selected soft weight is 0.01; primary seeds are 170-179.

Primary nonlinear feasibility metrics are
\begin{align}
\mathrm{CE}_{\mathrm{bilinear}}&=\operatorname{mean}_{i,s}|\Delta_i-RP_i|,\\
\mathrm{CE}_{\mathrm{simplex}}&=\operatorname{mean}_{s}\left|\sum_iP_i-1\right|,
\end{align}
together with maxima, range-conservation diagnostics, a joint $10^{-5}$ nonlinear violation rate, and reconstructed raw-$Q$ negative/order violation rates. Raw-$Q$ probabilistic metrics are primary.

\subsection{Predictive and constraint metrics}
MSE and MAE are computed from the predictive mean. Marginal CRPS is averaged across observations and coordinates, and multivariate Energy Score across observations. Central 90\% marginal interval coverage and width are reported. Circle additionally reports circular/geodesic MAE and the equality residual $y_1^2+y_2^2-1$. Affine FDC uses the six residuals $Q_i-Q_{i+1}+\Delta_i$; nonlinear FDC uses bilinear, simplex, range-conservation, nonnegativity, and reconstructed-order diagnostics. Boundary frequencies remain distributional diagnostics rather than feasibility metrics.

\subsection{Covariance ablation}
The covariance ablation uses the frozen linear benchmark and seeds 40-44. It compares the standard Projection-Or reference with \SLPProbHard{} using diagonal, rank-4 low-rank-plus-diagonal, and full covariance latent heads. The purpose is to test whether the diagonal latent law limits \SLPProbHard{} probabilistic expressiveness; it is not a matched full-covariance projection study.

\subsection{Controlled computational profiling}
Controlled profiling performs no training. Under a fixed CPU environment (8 threads, float32), it measures (i) the feasibility operator alone and (ii) a matched MLP + diagonal Gaussian sampling + feasibility operator pipeline. Each configuration uses 100 predictive samples, five warm-up calls, and 20 measured repetitions. Application-sized batches use $B=112$, with selected $B=1024$ throughput measurements. The five-family rerun includes affine equality, ordering, simplex, the nonlinear circle, and the real-world nonnegative-ordering geometry. The real-world affine/nonlinear numbers reported elsewhere are direct frozen-run implementation timings rather than additions to this controlled profiler.

\section{Additional Experimental Results}\label{app:additional-results}
For FDC comparisons, seed-based intervals and tests describe run variability conditional on the fixed dataset and temporal split. They do not quantify sampling uncertainty across new years, basins, or hydrological datasets.
\subsection{Float64 affine diagnostic}
\begin{table}[H]
\centering\small
\resizebox{0.98\linewidth}{!}{%
\begin{tabular}{lcccccc}
\toprule
Method & CEabs FP32 & CEabs FP64 & CEmax FP32 & CEmax FP64 & VR FP32 & VR FP64\\
\midrule
Projection-Or & $4.22\times 10^{-7}$ & $6.89\times 10^{-16}$ & $4.29\times 10^{-6}$ & $4.44\times 10^{-15}$ & 0.0000 & 0.0000 \\
Projection-Ob & $5.68\times 10^{-7}$ & $1.28\times 10^{-15}$ & $2.29\times 10^{-5}$ & $1.89\times 10^{-13}$ & 0.0017 & 0.0000 \\
\SLPProbHard{} & $2.69\times 10^{-7}$ & $4.97\times 10^{-16}$ & $2.86\times 10^{-6}$ & $5.33\times 10^{-15}$ & 0.0000 & 0.0000 \\
\bottomrule
\end{tabular}}
\caption{Seed-20 diagnostic. Trained models are unchanged; only feasibility-map / residual computation is repeated in float64.}
\label{tab:fp64}
\end{table}
The mean residuals collapse from approximately $10^{-7}$ in float32 to $10^{-15}$ in float64 for all three hard-constrained affine methods. Projection-Ob's seed-20 float32 violation rate of 0.0017 also disappears in float64. This supports the interpretation that the small base-case residuals are numerical realization effects rather than a relaxation of the mathematical constraints. The diagnostic is not sufficient by itself to attribute every high-dimensional scaling exceedance to the same source, so the manuscript keeps continuous residuals primary.

\subsection{Linear scaling}
Across the twelve pre-specified $(d,m,q)$ configurations, \SLPProbHard{} always uses $q=d-m$ stochastic coordinates, giving reductions from 12.3\% to 38.5\%. Point accuracy remains close. CRPS and Energy Score usually favor Projection-Or slightly, but the gap generally narrows at higher codimension. At larger configurations, float32 $10^{-5}$ threshold exceedances appear for both affine methods; \SLPProbHard{} is consistently lower in the largest cases. For $(d,m,q)=(97,33,64)$, Projection-Or has $\mathrm{VR}_{10^{-5}}=0.3958$ versus 0.0193 for SLP, while mean absolute residuals remain at $1.99\times10^{-6}$ and $1.23\times10^{-6}$, respectively. The complete table is in Appendix~\ref{app:scaling}.

\subsection{Ordering scaling}
From $d=3$ to $d=40$, positive-part \SLPProbHard{} has lower CRPS and Energy Score than end-to-end isotonic projection in every evaluated dimension. The CRPS gain grows from 0.55\% at $d=3$ to 25.96\% at $d=40$. More strikingly, isotonic projection's tolerance-based tie rate rises from 0.300 at $d=3$ to 0.891 at $d=40$, while the observed rate stays near 0.25. Positive-part \SLPProbHard{} remains close to that target from $d=5$ onward; averaged across the five scaling dimensions, its absolute tie-rate error is 0.030 versus 0.283 for end-to-end isotonic projection. Appendix~\ref{app:scaling} gives the complete scaling table.

\subsection{Synthetic composed nonnegative ordering: mechanism validation}
A controlled $d=5$ benchmark generated targets by cumulative positive-part, so both non-negativity and ordering are structurally active. Because the target generator shares the cumulative-positive-part mechanism with the \SLPProbHard{} map, this is a deliberately favorable controlled test of the boundary-preimage prediction rather than an independent test of generalization across unknown data-generating mechanisms. Across frozen seeds 115-119, both nonnegative isotonic projection and nonnegative cumulative positive-part \SLPProbHard{} achieve zero feasibility violations. Predictive scores are close, with \SLPProbHard{} slightly lower CRPS and Energy Score.

\begin{table}[H]
\centering\small
\resizebox{\linewidth}{!}{%
\begin{tabular}{lrrrrrr}
\toprule
Method & MSE & MAE & CRPS & ES & $P(y_1=0)$ & Pair-tie rate\\
\midrule
Projection-nonneg-isotonic & $0.48397\pm0.00507$ & $0.49376\pm0.00177$ & $0.35222\pm0.00122$ & $0.93289\pm0.00328$ & $0.5803$ & $0.5440$\\
SLP-nonnegative-positive-part & $0.48433\pm0.00751$ & $0.49437\pm0.00413$ & $0.35027\pm0.00315$ & $0.92047\pm0.00740$ & $0.5802$ & $0.4374$\\
Target & -- & -- & -- & -- & $0.5850$ & $0.4798$\\
\bottomrule
\end{tabular}}
\caption{Synthetic composed constraint benchmark, seeds 115-119. Both hard methods have zero negative values and zero ordering violations. \SLPProbHard{} is closer on adjacent-tie frequency, whereas projection is closer on the anchor-zero frequency.}
\label{tab:synthetic-b}
\end{table}

The paired Energy Score difference favors \SLPProbHard{} in all five seeds (paired $t$ $p=0.019$; Wilcoxon $p=0.0625$), while the paired test did not detect a CRPS difference at this five-seed level ($p=0.163$). With only five seeds, we report effect sizes, confidence intervals, and paired tests rather than an observed post-hoc power calculation. The absolute pair-tie calibration error is 0.0423 for \SLPProbHard{} versus 0.0643 for projection and favors \SLPProbHard{} in all five seeds ($p=0.00368$). Conversely, projection has the smaller anchor-zero error. This controlled benchmark therefore supports the boundary-preimage mechanism without implying universal boundary-calibration superiority.

\subsection{Nonlinear circle primary results}
\begin{table}[H]
\centering\scriptsize
\resizebox{\linewidth}{!}{%
\begin{tabular}{lrrrrrrrr}
\toprule
Method & MSE & MAE & CRPS & ES & Cov.90 & Width.90 & CEabs & VR$_{10^{-5}}$\\
\midrule
Unconstrained & 0.3622 & 0.4857 & 0.3422 & 0.5391 & 0.835 & 1.593 & 0.623 & 1.0000\\
Soft penalty & 0.3613 & 0.4850 & 0.3438 & 0.5417 & 0.810 & 1.484 & 0.584 & 1.0000\\
Post-hoc radial & 0.3649 & 0.4827 & 0.3344 & 0.5253 & 0.807 & 1.496 & $4.80\times10^{-8}$ & 0\\
End-to-end radial & 0.3557 & 0.4893 & 0.3277 & 0.5147 & 0.864 & 1.704 & $4.81\times10^{-8}$ & 0\\
SLP-angle & 0.3557 & 0.4836 & 0.3272 & 0.5143 & 0.854 & 1.611 & $1.70\times10^{-8}$ & 0\\
\bottomrule
\end{tabular}}
\caption{Unit-circle primary results, seeds 130-139. Residual differences between the two hard maps are numerical roundoff; both have zero $10^{-5}$ violations. \SLPProbHard{} uses one stochastic coordinate versus two for radial projection.}
\label{tab:circle-full}
\end{table}

For end-to-end radial projection versus SLP, paired tests detected no differences in MSE ($p=0.988$), CRPS ($p=0.750$), Energy Score ($p=0.888$), or circular MAE ($p=0.633$) at the ten-seed level. \SLPProbHard{} MAE is 1.17\% lower ($p=0.0279$). The correct primary conclusion is therefore practical hard feasibility with a 50\% stochastic-coordinate reduction and no detected differences in those principal scores, not universal score superiority.

\subsection{Real-world affine FDC results}\label{app:fdc-affine-results}
\begin{table}[H]
\centering\scriptsize
\resizebox{\linewidth}{!}{%
\begin{tabular}{lrrrrrrrr}
\toprule
Method & Raw MSE & Raw MAE & Raw CRPS & Raw ES & Cov.90 & Width.90 & CEabs & VR$_{10^{-5}}$\\
\midrule
Unconstrained & 104.47 & 4.738 & 3.689 & 13.882 & 0.766 & 13.86 & 0.1270 & 1.000\\
Soft affine & 104.37 & 4.731 & 3.686 & 13.882 & 0.763 & 13.62 & 0.1250 & 1.000\\
Post-hoc affine projection & 103.72 & 4.685 & 3.717 & 14.099 & 0.681 & 9.90 & $3.51\times10^{-8}$ & 0\\
End-to-end affine projection & 106.40 & 4.707 & 3.627 & 13.649 & 0.746 & 13.57 & $5.85\times10^{-8}$ & 0\\
SLP-affine & 105.14 & 4.649 & 3.698 & 13.997 & 0.712 & 10.94 & 0 & 0\\
\bottomrule
\end{tabular}}
\caption{Real-world affine FDC coherence, seeds 150-159. Raw-$Q$ metrics are primary; the affine branch enforces only $Q_i-Q_{i+1}+\Delta_i=0$, not nonnegativity/order.}
\label{tab:fdc-affine-full}
\end{table}

\begin{table}[H]
\centering\scriptsize
\resizebox{\linewidth}{!}{%
\begin{tabular}{lrrrrrr}
\toprule
Metric & Projection & \SLPProbHard{} & Mean diff. & 95\% CI & paired $t$ $p$ & Wilcoxon $p$\\
\midrule
Raw MSE & 106.398 & 105.136 & -1.262 & [-10.616, 8.092] & 0.767 & 0.770\\
Raw MAE & 4.707 & 4.649 & -0.058 & [-0.194, 0.078] & 0.361 & 0.432\\
Raw CRPS & 3.627 & 3.698 & +0.072 & [-0.032, 0.176] & 0.153 & 0.193\\
Raw ES & 13.649 & 13.997 & +0.348 & [-0.051, 0.747] & 0.080 & 0.105\\
Coverage.90 & 0.746 & 0.712 & -0.034 & [-0.061, -0.0068] & 0.0197 & 0.0195\\
Width.90 & 13.567 & 10.937 & -2.630 & [-3.075, -2.186] & $3.00\times10^{-7}$ & 0.00195\\
\bottomrule
\end{tabular}}
\caption{Paired real-world affine comparison: \SLPProbHard{} minus end-to-end affine projection. Paired tests did not detect differences in the principal accuracy/proper-score metrics at ten seeds; interval width and coverage show a significant sharpness-calibration trade-off.}
\label{tab:fdc-affine-paired}
\end{table}

The affine hard methods both have zero $10^{-5}$ violations. \SLPProbHard{} uses 7 rather than 13 stochastic coordinates (46.15\% reduction) and 37,774 versus 39,322 parameters. In the frozen implementation, median inference latency is 1.148 ms for \SLPProbHard{} and 1.623 ms for end-to-end projection; this approximately 29\% reduction is implementation-specific. The soft penalty selected at validation weight 0.01 retains $\mathrm{VR}_{10^{-5}}=1$, illustrating that the penalty baseline does not hard-enforce affine coherence.

\subsection{Real-world nonlinear scale-shape results}\label{app:fdc-nonlinear-results}
\begin{table}[H]
\centering\scriptsize
\resizebox{\linewidth}{!}{%
\begin{tabular}{lrrrrrrrr}
\toprule
Method & Raw MSE & Raw MAE & Raw CRPS & Raw ES & Cov.90 & Width.90 & Bilinear CE & Hard VR\\
\midrule
Unconstrained & 103.39 & 4.892 & 3.850 & 14.292 & 0.804 & 14.19 & 0.09695 & 1.000\\
Soft nonlinear & 103.93 & 4.971 & 3.921 & 14.488 & 0.805 & 14.62 & 0.09597 & 1.000\\
Post-hoc feasible correction & 102.33 & 4.840 & 3.568 & 13.754 & 0.823 & 14.97 & 0 & 0\\
End-to-end feasible correction & 104.92 & 5.151 & 4.177 & 14.483 & 0.630 & 15.75 & 0 & 0\\
SLP-scale-shape & 106.17 & 5.175 & 4.047 & 14.313 & 0.649 & 15.67 & 0 & 0\\
\bottomrule
\end{tabular}}
\caption{Real-world nonlinear scale-shape FDC coherence, seeds 170-179. All three hard methods also have zero negative reconstructed discharge and zero ordering violations.}
\label{tab:fdc-nonlinear-full}
\end{table}

\begin{table}[H]
\centering\scriptsize
\resizebox{\linewidth}{!}{%
\begin{tabular}{lrrrrrr}
\toprule
Metric & Correction & \SLPProbHard{} & Mean diff. & 95\% CI & paired $t$ $p$ & Wilcoxon $p$\\
\midrule
Raw MSE & 104.916 & 106.166 & +1.250 & [-3.778, 6.279] & 0.588 & 0.432\\
Raw MAE & 5.151 & 5.175 & +0.024 & [-0.175, 0.223] & 0.790 & 0.432\\
Raw CRPS & 4.177 & 4.047 & -0.130 & [-0.285, 0.025] & 0.090 & 0.232\\
Raw ES & 14.483 & 14.313 & -0.171 & [-0.609, 0.268] & 0.402 & 0.557\\
Coverage.90 & 0.630 & 0.649 & +0.020 & [-0.066, 0.105] & 0.619 & 0.375\\
Width.90 & 15.747 & 15.669 & -0.078 & [-1.909, 1.753] & 0.925 & 1.000\\
\bottomrule
\end{tabular}}
\caption{Paired real-world nonlinear comparison: \SLPProbHard{} minus the five-iteration end-to-end alternating feasible correction. Paired tests did not detect differences in the principal predictive metrics at ten seeds.}
\label{tab:fdc-nonlinear-paired}
\end{table}

The post-hoc comparator is also feasible and has lower reported mean raw-$Q$ CRPS than both end-to-end alternatives. Table~\ref{tab:fdc-nonlinear-paired} compares SLP with the end-to-end correction model only; it does not provide a paired significance test against the post-hoc model. These results do not identify the cause of the performance difference.

Both primary end-to-end hard methods have zero bilinear violations at $10^{-5}$, simplex residuals near float32 roundoff, zero negative reconstructed discharge, and zero FDC ordering violations. \SLPProbHard{} uses 8 computational stochastic coordinates versus 14 ambient coordinates for the correction model (42.86\% reduction), with one redundant computational coordinate relative to the 7 intrinsic DOF. Parameter counts are 38,032 versus 39,580. Frozen median inference latency is 1.478 ms for \SLPProbHard{} versus 7.369 ms for the five-iteration end-to-end correction; the roughly 80\% reduction is specific to these implementations and iteration count. The validation-selected soft nonlinear baseline has bilinear CE about 0.096 and hard violation rate 1.0.

\subsection{Complete paired statistics}
\begin{table}[H]
\centering\scriptsize
\resizebox{\linewidth}{!}{%
\begin{tabular}{llrrrrrr}
\toprule
Family & Metric & Reference mean & \SLPProbHard{} mean & Rel. \SLPProbHard{} diff. & 95\% CI of paired diff. & paired $t$ $p$ & Wilcoxon $p$\\
\midrule
Linear equality & mse & 0.29646 & 0.29574 & -0.25\% & [-0.003834, 0.00238] & 0.609 & 0.77 \\
Linear equality & mae & 0.39286 & 0.39295 & +0.02\% & [-0.001766, 0.001944] & 0.916 & 0.77 \\
Linear equality & crps & 0.28467 & 0.29012 & +1.91\% & [0.003935, 0.00696] & 1.91e-05 & 0.00195 \\
Linear equality & energy score & 1.20878 & 1.22158 & +1.06\% & [0.005968, 0.01964] & 0.00218 & 0.00391 \\
Weak ordering & mse & 5.56771 & 5.09482 & -8.49\% & [-1.049, 0.1034] & 0.0964 & 0.084 \\
Weak ordering & mae & 1.09838 & 1.08457 & -1.26\% & [-0.02347, -0.004143] & 0.0103 & 0.0195 \\
Weak ordering & crps & 0.79450 & 0.78204 & -1.57\% & [-0.01804, -0.006897] & 0.000678 & 0.00195 \\
Weak ordering & energy score & 2.24397 & 2.18828 & -2.48\% & [-0.07136, -0.04002] & 2.13e-05 & 0.00195 \\
Simplex & mse & 0.03222 & 0.03196 & -0.84\% & [-0.0004537, -7.085e-05] & 0.0127 & 0.00977 \\
Simplex & mae & 0.11799 & 0.11833 & +0.28\% & [-0.0002802, 0.000968] & 0.244 & 0.232 \\
Simplex & crps & 0.08899 & 0.08662 & -2.67\% & [-0.002826, -0.00191] & 9.62e-07 & 0.00195 \\
Simplex & energy score & 0.24740 & 0.24306 & -1.76\% & [-0.005441, -0.003236] & 9.33e-06 & 0.00195 \\
\bottomrule
\end{tabular}}
\caption{Paired untouched-seed comparisons. Relative difference is $100(\mathrm{SLP}-\mathrm{reference})/\mathrm{reference}$; negative values favor \SLPProbHard{} for error/score metrics. $p$-values are unadjusted and are reported as diagnostics rather than as a multiple-comparison discovery exercise.}
\label{tab:paired}
\end{table}

\subsection{Latent covariance expressiveness}
Table~\ref{tab:cov-ablation} tests whether the diagonal latent Gaussian is responsible for the linear \SLPProbHard{} probabilistic-score gap. All variants remain feasible at the $10^{-5}$ threshold. Moving from diagonal \SLPProbHard{} to low-rank or full covariance leaves MSE/MAE essentially unchanged but improves CRPS from 0.28887 to about 0.2837 and Energy Score from 1.21714 to about 1.20710. The low-rank and full variants are numerically very close in predictive score on this small $q=8$ benchmark. Because the rank-4 implementation has slightly more parameters than the full Cholesky head at this dimension, no parameter-efficiency claim is made for low-rank covariance here.

\begin{table}[H]
\centering\small
\begin{tabular}{lrrrrrr}
\toprule
Method & MSE & CRPS & ES & Cov.90 & Parameters & $\mathrm{VR}_{10^{-5}}$\\
\midrule
Projection-Or & 0.29741 & 0.28544 & 1.21124 & 0.7825 & 37,526 & 0\\
\SLPProbHard{} diagonal & 0.29543 & 0.28887 & 1.21714 & 0.7954 & 36,752 & 0\\
\SLPProbHard{} low-rank & 0.29591 & 0.28364 & 1.20710 & 0.8016 & 40,880 & 0\\
\SLPProbHard{} full & 0.29587 & 0.28376 & 1.20710 & 0.7969 & 40,364 & 0\\
\bottomrule
\end{tabular}
\caption{Covariance expressiveness ablation, seeds 40-44. The experiment compares covariance families within \SLPProbHard{} against the standard Projection-Or reference; it is not a matched full-covariance projection study.}
\label{tab:cov-ablation}
\end{table}

\section{Controlled Computational Profiling}\label{app:profiling}
\subsection{Controlled computational profiling}
Table~\ref{tab:profiling} reports the circle-inclusive controlled rerun for the application-sized $B=112$, $S=100$ end-to-end workload. Affine equality shows a 1.30$\times$ \SLPProbHard{} advantage. Simplex savings rise from 1.06$\times$ at $d=5$ to 2.12$\times$ at $d=40$. On the nonlinear circle, \SLPProbHard{} is 1.22$\times$ faster end to end while sampling one structural coordinate rather than two ambient coordinates. Ordering remains dramatically faster with cumulative positive-part than the project's reference PAVA/isotonic implementation; the real-world nonnegative-ordering pair shows the same implementation pattern. These ratios are evidence about this controlled implementation and workload, not universal solver speedups.

\begin{table}[H]
\centering\small
\setlength{\tabcolsep}{3pt}
\begin{tabular}{lrrrrr}
\toprule
Family & $d$ & \makecell{Reference\\ms} & \makecell{\SLPProbHard{}\\ms} & \makecell{Ref./\\\SLPProbHard{}} & \makecell{Latency\\reduction}\\
\midrule
Linear equality & 11 & 1.784 & 1.367 & 1.30$\times$ & 23.4\%\\
Ordering & 5 & 248.802 & 0.958 & 259.7$\times$ & 99.61\%\\
Ordering & 40 & 504.658 & 3.872 & 130.3$\times$ & 99.23\%\\
Simplex & 5 & 1.471 & 1.382 & 1.06$\times$ & 6.0\%\\
Simplex & 40 & 7.492 & 3.540 & 2.12$\times$ & 52.8\%\\
Circle & 2 & 0.616 & 0.503 & 1.22$\times$ & 18.3\%\\
Real-world nonneg. order & 7 & 272.462 & 0.949 & 287.1$\times$ & 99.65\%\\
\bottomrule
\end{tabular}
\caption{Controlled end-to-end CPU profiling: matched MLP + diagonal Gaussian sampling + feasibility operation, batch 112 and 100 predictive samples, 8 threads, float32, 5 warm-ups and 20 measured repeats. Ordering/PAVA ratios are implementation-level.}
\label{tab:profiling}
\end{table}

For the circle, isolated operator latency is 0.149 ms for radial normalization and 0.100 ms for $(\cos z,\sin z)$ at $B=112$, a 33.0\% reduction. At $B=1024$, end-to-end circle latency is 3.111 ms versus 2.359 ms (1.32$\times$). Both circle operators satisfy the equality to numerical precision. For affine equality, the $B=1024$ end-to-end pipeline is 13.564 ms versus 10.153 ms (1.34$\times$), illustrating that lower stochastic dimension can matter even when both feasibility operators are simple tensor operations.

\section{Complete Scaling Tables}\label{app:scaling}
\subsection{Linear equality scaling}\label{app:linear-scaling}
In Table~\ref{tab:linear-scaling}, $\Delta$ is $100(\mathrm{SLP}-\mathrm{Projection})/\mathrm{Projection}$; negative predictive-score values favor SLP. The scaling tables report predictive, feasibility, boundary, and
representation-dimensional behavior only. Runtime comparisons are reported
separately under the controlled profiling protocol in Appendix~\ref{app:profiling}.

\begin{table}[H]
\centering\scriptsize
\resizebox{\linewidth}{!}{%
\begin{tabular}{rrrrrrrrrrr}
\toprule
$d$ & $m$ & $q$ & Dim.red. & $\Delta$MSE & $\Delta$CRPS & $\Delta$ES & VR Proj. & VR \SLPProbHard{} & CEabs Proj. & CEabs \SLPProbHard{} \\
\midrule
10 & 2 & 8 & 20.0\% & +0.46\% & +3.56\% & +2.00\% & 0.0000 & 0.0000 & $2.81\times 10^{-7}$ & $3.55\times 10^{-7}$ \\
11 & 3 & 8 & 27.3\% & +0.71\% & +2.06\% & +1.22\% & 0.0000 & 0.0000 & $4.13\times 10^{-7}$ & $2.63\times 10^{-7}$ \\
13 & 5 & 8 & 38.5\% & -0.56\% & +0.58\% & +0.36\% & 0.0000 & 0.0000 & $3.78\times 10^{-7}$ & $2.19\times 10^{-7}$ \\
19 & 3 & 16 & 15.8\% & +0.03\% & +2.45\% & +1.11\% & 0.0000 & 0.0000 & $5.53\times 10^{-7}$ & $6.11\times 10^{-7}$ \\
21 & 5 & 16 & 23.8\% & -0.57\% & +1.08\% & +0.38\% & 0.0000 & 0.0000 & $4.19\times 10^{-7}$ & $4.00\times 10^{-7}$ \\
25 & 9 & 16 & 36.0\% & -0.19\% & +0.46\% & +0.36\% & 0.0000 & 0.0000 & $4.78\times 10^{-7}$ & $3.91\times 10^{-7}$ \\
37 & 5 & 32 & 13.5\% & +0.95\% & +1.51\% & +1.03\% & 0.0000 & 0.0000 & $9.42\times 10^{-7}$ & $7.75\times 10^{-7}$ \\
41 & 9 & 32 & 22.0\% & +0.02\% & +0.50\% & +0.43\% & 0.0021 & 0.0000 & $9.67\times 10^{-7}$ & $7.95\times 10^{-7}$ \\
49 & 17 & 32 & 34.7\% & -0.37\% & +0.16\% & +0.27\% & 0.0007 & 0.0000 & $6.73\times 10^{-7}$ & $5.34\times 10^{-7}$ \\
73 & 9 & 64 & 12.3\% & -0.38\% & +0.60\% & +0.10\% & 0.0577 & 0.0015 & $1.71\times 10^{-6}$ & $1.60\times 10^{-6}$ \\
81 & 17 & 64 & 21.0\% & -0.16\% & +0.04\% & +0.16\% & 0.1672 & 0.0000 & $1.47\times 10^{-6}$ & $9.36\times 10^{-7}$ \\
97 & 33 & 64 & 34.0\% & -0.10\% & +0.06\% & +0.05\% & 0.3958 & 0.0193 & $1.99\times 10^{-6}$ & $1.23\times 10^{-6}$ \\
\bottomrule
\end{tabular}}
\caption{Twelve held-out linear scaling configurations, seeds 30-34.}
\label{tab:linear-scaling}
\end{table}

\subsection{Ordering scaling}
Positive gain values in Table~\ref{tab:order-scaling} mean lower/better \SLPProbHard{} error or score. Boundary-error columns are absolute differences from the observed tolerance-based tie rate.
\begin{table}[H]
\centering\scriptsize
\resizebox{\linewidth}{!}{%
\begin{tabular}{rrrrrrrrrrr}
\toprule
$d$ & MSE gain & CRPS gain & ES gain & Obs.tie & \SLPProbHard{} tie & Proj.tie & \SLPProbHard{} tie err. & Proj tie err.\\
\midrule
3 & +5.25\% & +0.55\% & +1.14\% & 0.261 & 0.181 & 0.300 & 0.079 & 0.040 \\
5 & +17.78\% & +2.35\% & +3.05\% & 0.251 & 0.235 & 0.401 & 0.017 & 0.149 \\
10 & +11.65\% & +3.91\% & +4.79\% & 0.255 & 0.234 & 0.504 & 0.021 & 0.249 \\
20 & +10.41\% & +3.85\% & +4.48\% & 0.254 & 0.232 & 0.591 & 0.022 & 0.337 \\
40 & +17.75\% & +25.96\% & +23.50\% & 0.253 & 0.241 & 0.891 & 0.012 & 0.638 \\
\bottomrule
\end{tabular}}
\caption{Ordering scaling, seeds 70-74.}
\label{tab:order-scaling}
\end{table}

\subsection{Simplex scaling}
\begin{table}[H]
\centering\scriptsize
\resizebox{\linewidth}{!}{%
\begin{tabular}{rrrrrrrrrrr}
\toprule
$d$ & MSE gain & CRPS gain & ES gain & Obs.zero & \SLPProbHard{} zero & Proj.zero & \SLPProbHard{} zero err. & Proj zero err. \\
\midrule
3 & +0.85\% & +2.80\% & +2.50\% & 0.245 & 0.164 & 0.195 & 0.081 & 0.050 \\
5 & +1.51\% & +2.80\% & +1.87\% & 0.250 & 0.164 & 0.233 & 0.086 & 0.017 \\
10 & +0.82\% & +3.38\% & +1.54\% & 0.250 & 0.253 & 0.325 & 0.003 & 0.075 \\
20 & +2.53\% & +4.47\% & +2.16\% & 0.252 & 0.338 & 0.404 & 0.086 & 0.152 \\
40 & +1.99\% & +4.76\% & +1.82\% & 0.252 & 0.372 & 0.453 & 0.120 & 0.202 \\
\bottomrule
\end{tabular}}
\caption{Simplex scaling, seeds 110-114. Positive gains mean lower \SLPProbHard{} score.}
\label{tab:simplex-scaling}
\end{table}

\section{Extended Discussion and Limitations}
\subsection{The feasible representation is part of the probabilistic model}\label{app:discussion-representation}
The central empirical conclusion is that an \SFLP{} inside \SLPProbHard{} is not merely a constraint-enforcement wrapper. Because $P^Y(B)=P^Z(h^{-1}(B))$, changing $h$ changes support geometry and probability allocation. The weak-order and nonnegative-order maps illustrate this distinction: weak-order cumulative positive-part supports adjacent ties but not a positive anchor-zero atom under a continuous anchor, whereas nonnegative cumulative positive-part gives positive-probability preimages to both. Both constructions are feasible for the constraints they encode, yet they define different predictive-law families.

\paragraph{Hydrological interpretation.} The hydrological experiments are deliberately constructed from one frozen dataset. Nonnegative ordering asks whether every FDC sample lies in the physical order cone; affine coherence asks whether explicitly represented increments agree with adjacent FDC differences; nonlinear scale-shape asks whether increment magnitude and compositional shape satisfy $\Delta=RP$. The underlying data and split are held fixed, with the representation-specific target transformations described in Section 4. This makes the new real-world evidence a test of the central claim that feasible representation is part of probabilistic model design.

\subsection{Dimensionality is family- and map-dependent}
The affine equality case provides the clean intrinsic-coordinate result: $q=d-m$ is both the mathematical dimension of the feasible affine set and the \SLPProbHard{} stochastic dimension. The primary synthetic experiment reduces 11 stochastic coordinates to 8, the circle reduces two ambient coordinates to one structural angle, and the real-world affine FDC representation reduces 13 ambient stochastic coordinates to 7. This experiment is not expected to establish universal score superiority: under diagonal covariance Projection-Or scores better, while richer \SLPProbHard{} covariance closes or reverses that gap without changing feasibility or the $q$-coordinate representation. By contrast, the current weak-order map uses $d$ latents for a $d$-dimensional cone, normalized positive-part boundary-mass simplex uses $d$ computational latents for a $d-1$ dimensional simplex, and the nonlinear real-world scale-shape map uses 8 computational latents for a 7-DOF structural object. The safe general claim is that \SLPProbHard{} \emph{can exploit intrinsic coordinates when the map permits}, not that every \SLPProbHard{} map is lower-dimensional.

\subsection{Boundary capability is not boundary calibration}
Boundary capability does not imply boundary calibration. In the primary $d=5$ simplex experiment, projection's zero rate is closer to the observed rate, whereas SLP is closer on the reported primary ordering tie rate. The map image determines reachability, while the latent probabilities of event preimages determine boundary mass.

\subsection{Covariance expressiveness and structural feasibility are separate design axes}
The covariance ablation shows that richer latent dependence can improve probabilistic scores without altering the structural feasibility guarantee. This supports a modular interpretation of \SLPProbHard{}: the structural map specifies feasible images and coordinates, while the latent law allocates probability through its preimages; together they constitute the \SFLP{}. The diagonal Gaussian is therefore a modeling choice rather than a defining limitation of \SLPProbHard{}. A matched full-covariance projection comparison remains outside this particular ablation.

\subsection{Computation depends on the constraint mechanism and implementation}\label{app:discussion-computation}
Controlled profiling replaces the earlier training-log timing observations. The affine case shows that structural generation need not always make an isolated operator faster, although lower stochastic dimension can still reduce end-to-end workload. Simplex advantages grow with dimension but are moderated by shared neural inference. Ordering shows very large gains under the project's PAVA implementation because cumulative positive-part is a direct tensor operation while isotonic projection processes constrained rows. The correct claim is therefore family- and implementation-dependent computational advantage, not a universal complexity theorem for \SLPProbHard{} versus projection.

\paragraph{Architectural and timing claims.} The primary computational statement is architectural: $11\!\to\!8$ stochastic coordinates is representation-level and hardware independent. Wall-clock timing is secondary evidence. Under the present matched CPU implementation \SLPProbHard{} has lower latency, but optimized affine conditioning is itself inexpensive and other families depend strongly on implementation. Accordingly, all runtime statements are qualified by the exact workload and reference code.

\subsection{Relationship to ProbHardE2E/DPPL}
The completed affine experiment confirms a complementary rather than uniformly ordered relationship. ProbHardE2E/DPPL conditions an ambient Gaussian, while affine \SLPProbHard{} learns the $q=d-m$ structural Gaussian whose pushforward is feasible by construction. On the matched ten-seed benchmark, both are hard feasible at $10^{-5}$; \SLPProbHard{} has lower MSE/MAE and a 27.27\% smaller stochastic representation, whereas DPPL has lower marginal CRPS and better 90\% coverage, and the paired test did not detect an Energy Score difference. This comparison is direct but deliberately scoped: it uses the frozen official source through a verified affine adapter and does not reproduce the authors' PDE application or establish dominance on other constraint classes.

\section{Limitations}
\begin{enumerate}
\item \textbf{Explicit map requirement.} Some feasible sets are easier to project onto than to parameterize, and projection is preferable in those cases.
\item \textbf{Map-dependent coverage and calibration.} A non-surjective map introduces parameterization bias, while even a boundary-capable map does not guarantee correct boundary frequencies. The present experiments do not resolve the source of the observed undercoverage or establish nominal calibration.
\item \textbf{Map-dependent dimensionality.} Only some structural maps achieve intrinsic computational coordinates; normalized positive-part simplex is intentionally redundant by one coordinate.
\item \textbf{Latent distribution remains a modeling choice.} Richer covariance improves \SLPProbHard{} probabilistic scores in the linear ablation, but the present study does not exhaust non-Gaussian or flow-based latent families.
\item \textbf{Nonsmooth and discontinuous maps introduce optimization effects.} The positive-part map is continuous but nonsmooth and has a flat negative region. The normalized positive-part map additionally uses a discontinuous argmax-vertex fallback on the all-inactive region. The latter requires a separate analysis of expected-score gradient estimation.
\item \textbf{Exact feasibility and analytic tractability differ.} Nonlinear pushforward densities or moments may be unavailable even when every sample is feasible.
\item \textbf{Nonlinear-equality coverage is explicit rather than general.} The unit circle and scale-shape FDC supply clean nonlinear equalities, but no arbitrary nonlinear-equality parameterization or higher-dimensional manifold scaling is claimed.
\item \textbf{Controlled timing is implementation-specific.} In particular, the very large ordering gaps reflect the project's CPU PAVA/isotonic implementation, and the nonlinear FDC comparison uses a five-iteration alternating correction; these should not be read as universal speedups over every optimized implementation.
\item \textbf{The hydrological target is model-based reanalysis discharge.} The seven-basin benchmark uses GloFAS v5.0 LISFLOOD-based reanalysis discharge rather than observed-gauge discharge~\citep{grimaldi2026glofas}; it validates real-world constraint structure and distributional behavior rather than observational hydrological skill.
\item \textbf{The direct ProbHardE2E comparison is affine-adapter scoped.} The official repository and released Gaussian conditioning equations are frozen and verified, but the released PDE wrapper does not expose arbitrary $A,b$. We therefore substitute the benchmark hierarchy matrix while retaining the released equations; this is not a reproduction of the original PDE datasets or published runtime.
\item \textbf{Finite precision remains visible.} Affine real-arithmetic guarantees can leave small float32 residuals; the primary float64 diagnostic supports this interpretation for the base case but does not establish the cause of every scaling exceedance.
\end{enumerate}

\section{Reproducibility and Reporting Conventions}
\paragraph{Code and data availability.}
The accompanying arXiv source package contains the manuscript, bibliography, and document styles. Experimental code, per-seed result archives, and processed hydrological data are separate research artifacts; a public repository location is not specified in this version.

All frozen training-result directories were generated on the HPC CPU partition using the validated Python 3.11 / PyTorch CPU environment. The environment check used the exact Conda interpreter, disabled user-site packages, and forced the environment's C++ runtime to avoid system-library contamination. Final jobs used eight allocated CPU threads. Result bundles preserve per-seed CSVs, configuration JSON files, histories, aggregate summaries, paper tables, checkpoints where enabled, and the float64 diagnostic.

Continuous feasibility residuals remain primary. Thresholded violation rates at $10^{-5}$ are finite-precision diagnostics. Ordering tie and simplex zero rates are distributional boundary diagnostics and are not conflated with violation metrics. The real-world affine and nonlinear extensions reuse the same frozen source artifact and preserve its SHA-256 in their run tables; raw-$Q$ metrics are primary for hydrological interpretation. The official-source DPPL affine-adapter comparison preserves the upstream Git SHA and source hashes, verifies the generic adapter against a direct transcription of the released equations, reuses the frozen \SLPProbHard{} rows without retraining, and profiles only after the predictive primary is frozen. Paired $p$-values are unadjusted and interpreted alongside effect sizes.

\section{Scope of Empirical Claims}
The evidence supports \SLPProbHard{} for constraint classes admitting suitable explicit \SFLP{}s: direct feasible generation, exact affine free coordinates, representation-dependent boundary probability, one-coordinate circular generation, real-world affine and nonlinear coherence, composed nonnegative/ordered discharge generation, covariance-head modularity, and family-dependent computational differences under the stated implementations. The official-source ProbHardE2E evidence is a direct ten-seed comparison on one compatible affine benchmark; it supports the reported representation trade-off there but does not establish superiority across ProbHardE2E's nonlinear or application-specific constraint classes. Generality remains map dependent, and arbitrary nonlinear equalities still require suitable explicit maps or other enforcement mechanisms.

\section{Official ProbHardE2E/DPPL Affine Comparison Audit}\label{app:dppl-audit}

\subsection{Frozen source provenance and adapter scope}\label{app:dppl-source}
The direct comparator uses the official \texttt{amazon-science/probharde2e} repository~\citep{utkarsh2025,amazonprobharde2e2026} frozen at Git commit
\begin{center}
\texttt{ff787718e79e5ac6b673807d40ce1cddb0fd17c2}
\end{center}
(commit date 2026-03-07, branch \texttt{main}). The worktree was clean. SHA-256 hashes were recorded for \texttt{probconserv.py}, \texttt{ProbHardE2E\_Linear.ipynb}, \texttt{README.md}, and \texttt{LICENSE}; the run also checked for the released hard-linear equation fragments and the linear-notebook calls to the probabilistic conservation layer.

The released wrapper~\citep{utkarsh2025,amazonprobharde2e2026} constructs its PDE-specific conservation matrix $g$ internally and does not expose the arbitrary affine interface needed for the hierarchy benchmark. The adapter therefore makes one problem-interface substitution,
\[
g_{\mathrm{PDE}}\longrightarrow A_{\mathrm{benchmark}},
\]
while retaining the released Gaussian conditioning equations. It is consequently described throughout as \emph{ProbHardE2E/DPPL (official-source affine adapter)}, not as an untouched call to the original PDE wrapper.

The frozen comparison record identifies the affine-adapter source in the historical v24 bundle as \texttt{frozen\_dppl\_evidence/dppl\_official\_adapter.py}. Its recorded SHA-256 is
\begin{center}
\texttt{442cdbbe3dbd58ef0a8abbe1824baf09108a72e0ab258dfffc8438feb6f10087},
\end{center}
The historical archive path and recorded hash document provenance; they do not establish public availability. The adapter file and its machine-readable provenance record are separate research artifacts and are not included in the accompanying manuscript-source package.
This provenance description does not alter or rerun the frozen predictive experiment.

\subsection{Affine DPPL equations and coherent sampling}\label{app:dppl-equations}
For an ambient diagonal Gaussian $\widetilde Y\mid x\sim\mathcal N(\mu,\Sigma)$ with affine equality $AY=b$, the verified adapter retains the released ProbHardE2E/DPPL hard-linear conditioning equations~\citep{utkarsh2025,amazonprobharde2e2026}:
\begin{align}
S &= A\Sigma A^\top,\qquad
K=\Sigma A^\top S^{-1},\\
\mu_c &= \mu-K(A\mu-b),\\
\Sigma_c &= \Sigma-KA\Sigma.
\end{align}
Because the common objective and evaluation use sample-based CRPS and Energy Score, independent sampling from only the projected marginal variances would discard the induced cross-coordinate covariance. Samples are therefore generated by the equivalent reparameterization
\begin{equation}
Y=\widetilde Y-K(A\widetilde Y-b),
\qquad \widetilde Y\sim\mathcal N(\mu,\Sigma),
\end{equation}
which preserves the singular conditioned joint Gaussian in exact arithmetic. The constraint solve is performed in float64 and returned to the model dtype, matching the released source's double-precision constraint path.

\subsection{Equation-level equivalence and finite-precision audit}
Before primary evaluation, the generic affine adapter was compared with a direct transcription of the released hard-linear equations using the same benchmark $A,b$. Table~\ref{tab:dppl-equivalence} reports the frozen check.
\begin{table}[H]
\centering\small
\begin{tabular}{lr}
\toprule
Audit quantity & Result\\
\midrule
Adapter vs.\ direct-reference mean, max abs.\ diff. & $0$\\
Adapter vs.\ direct-reference covariance, max abs.\ diff. & $0$\\
Conditioned-mean constraint max abs.\ residual & $1.19\times10^{-7}$\\
$A\Sigma_c$ max abs.\ residual & $4.84\times10^{-8}$\\
Float64 sample constraint max abs.\ residual & $2.66\times10^{-15}$\\
Float32 sample CE$_{\mathrm{abs}}$ & $4.60\times10^{-8}$\\
Float32 sample CE$_{\max}$ & $5.36\times10^{-7}$\\
Float32 VR$_{10^{-5}}$ & $0$\\
\bottomrule
\end{tabular}
\caption{Frozen equivalence and feasibility audit for the official-source affine adapter.}
\label{tab:dppl-equivalence}
\end{table}

\subsection{Pre-specified protocol and execution audit}\label{app:dppl-freeze}
Before inspecting official-source DPPL affine-adapter test results, the direct-comparator protocol froze the existing affine task, seed registry 20-29, common metrics, float32 evaluation with 100 predictive samples, and the rule that existing \SLPProbHard{} results would not be retrained or retuned after comparator inspection. The executed benchmark uses 12 inputs, 11 outputs, rank-3 affine constraints, and intrinsic dimension 8, with $3000/750/750$ train/validation/test samples; both methods use the same width-128, depth-3 backbone, batch size 128, 60-epoch cap, learning rate $10^{-3}$, weight decay $10^{-5}$, gradient clip 5, 12/32 train/validation Monte Carlo samples, and 0.5 CRPS + 0.5 Energy Score objective.

The execution audit confirms all ten primary seeds, DPPL stochastic dimension 11 versus \SLPProbHard{} dimension 8, and no \SLPProbHard{} retraining. The frozen \SLPProbHard{} result file retained the same SHA-256 before and after the DPPL job. Seed-specific train/validation/test tensor hashes were stored for every DPPL seed, and profiling was run only after the predictive primary had been frozen.

\subsection{Full direct-comparison results}
Table~\ref{tab:dppl-full} gives the complete primary means and standard deviations. Both methods have zero $10^{-5}$ violations. DPPL's smaller float32 residuals and SLP's larger ones are both far below the diagnostic threshold and are not interpreted as a substantive feasibility difference.
\begin{table}[H]
\centering\small
\resizebox{\linewidth}{!}{%
\begin{tabular}{lrr}
\toprule
Metric & DPPL source adapter & \SLPProbHard{}\\
\midrule
MSE & $0.29968\pm0.00797$ & $\mathbf{0.29574\pm0.00614}$\\
MAE & $0.39466\pm0.00479$ & $\mathbf{0.39295\pm0.00360}$\\
Marginal CRPS & $\mathbf{0.28676\pm0.00367}$ & $0.29012\pm0.00277$\\
Energy Score & $1.21896\pm0.01550$ & $1.22158\pm0.01228$\\
Coverage 90 & $\mathbf{0.8093\pm0.0096}$ & $0.7856\pm0.0078$\\
Interval width 90 & $1.2967\pm0.0358$ & $\mathbf{1.2637\pm0.0232}$\\
CE$_{\mathrm{abs}}$ & $6.55\times10^{-8}$ & $2.60\times10^{-7}$\\
CE$_{\max}$ & $9.54\times10^{-7}$ & $2.59\times10^{-6}$\\
VR$_{10^{-5}}$ & $0$ & $0$\\
Stochastic coordinates & $11$ & $\mathbf{8}$\\
Trainable parameters & $37{,}526$ & $\mathbf{36{,}752}$\\
Primary median inference latency (ms) & $3.222\pm0.136$ & $\mathbf{1.647\pm0.191}$\\
\bottomrule
\end{tabular}}
\caption{Official-source ProbHardE2E/DPPL affine-adapter comparison versus the already-frozen \SLPProbHard{} primary, seeds 20-29. Bold marks lower error/score, coverage closer to the nominal 0.90 level, smaller coordinate/parameter counts, and lower measured latency. Interval width is a sharpness diagnostic and must be interpreted jointly with coverage; latency is implementation-specific.}
\label{tab:dppl-full}
\end{table}

\subsection{Paired statistical interpretation}
Table~\ref{tab:dppl-paired} reports \SLPProbHard{} minus DPPL paired differences. MSE and MAE favor SLP; marginal CRPS and coverage favor DPPL; the paired test did not detect an Energy Score difference. The interval result is a sharpness-coverage trade-off rather than an unconditional improvement.
\begin{table}[H]
\centering\scriptsize
\resizebox{\linewidth}{!}{%
\begin{tabular}{lrrrrrr}
\toprule
Metric & DPPL mean & \SLPProbHard{} mean & SLP$-$DPPL & 95\% CI & paired $t$ $p$ & Wilcoxon $p$\\
\midrule
MSE & 0.29968 & 0.29574 & $-0.00395$ & $[-0.00653,-0.00137]$ & 0.0072 & 0.0098\\
MAE & 0.39466 & 0.39295 & $-0.00171$ & $[-0.00328,-0.00015]$ & 0.0353 & 0.0645\\
CRPS & 0.28676 & 0.29012 & $+0.00336$ & $[0.00198,0.00474]$ & $3.76\times10^{-4}$ & 0.0020\\
Energy Score & 1.21896 & 1.22158 & $+0.00263$ & $[-0.00397,0.00922]$ & 0.3913 & 0.5566\\
Coverage 90 & 0.8093 & 0.7856 & $-0.02367$ & $[-0.03246,-0.01488]$ & $1.81\times10^{-4}$ & 0.0039\\
Width 90 & 1.2967 & 1.2637 & $-0.03302$ & $[-0.06293,-0.00310]$ & 0.0341 & 0.0371\\
CE$_{\mathrm{abs}}$ & $6.55\times10^{-8}$ & $2.60\times10^{-7}$ & $+1.94\times10^{-7}$ & $[1.86,2.03]\times10^{-7}$ & $1.93\times10^{-12}$ & 0.0020\\
CE$_{\max}$ & $9.54\times10^{-7}$ & $2.59\times10^{-6}$ & $+1.63\times10^{-6}$ & $[1.42,1.85]\times10^{-6}$ & $3.16\times10^{-8}$ & 0.0020\\
\bottomrule
\end{tabular}}
\caption{Paired direct-comparator statistics across the same ten seeds. $p$-values are descriptive and unadjusted. Residual differences are numerical-scale effects: both methods have VR$_{10^{-5}}=0$.}
\label{tab:dppl-paired}
\end{table}

\subsection{Stochastic dimension, covariance family, and computation}
The direct comparison makes the representation difference explicit. DPPL starts from an 11-coordinate ambient diagonal Gaussian and conditions it onto the rank-3 affine set. \SLPProbHard{} instead parameterizes the 8 free coordinates directly. Hence the structural representation removes
\[
100\left(1-\frac{8}{11}\right)=27.27\%
\]
of stochastic coordinates before feasibility is considered. These two models need not induce the same covariance family on the feasible subspace: DPPL's covariance weighting can exploit ambient marginal variances, whereas diagonal affine \SLPProbHard{} induces covariance through $N\Sigma_ZN^\top$. This distinction is consistent with the observed DPPL marginal-CRPS advantage and the separate finding that richer \SLPProbHard{} covariance heads improve CRPS/ES without changing feasibility.

The primary run's matched instrumentation gives median inference latency $3.222\pm0.136$ ms for DPPL versus $1.647\pm0.191$ ms for SLP. A separate DPPL-only frozen profile, performed after the predictive primary and without retraining, uses batch 128, 100 samples, eight CPU threads, 20 warmups, and 200 repeats; its per-seed median averages $3.115\pm0.116$ ms. These values describe the present adapter/\SLPProbHard{} implementations and hardware only; they are not reproductions of the runtime reported by the original ProbHardE2E paper.

\subsection{What the direct comparison does and does not establish}
The result closes the compatible affine-comparator gap: \SLPProbHard{} remains practically hard feasible against the verified official-source ProbHardE2E affine adapter while using fewer stochastic coordinates; predictive advantages are metric dependent. It does \emph{not} show that \SLPProbHard{} dominates ProbHardE2E across all metrics or constraint classes. DPPL has the better marginal CRPS and higher, closer-to-nominal coverage here; \SLPProbHard{} has better MSE/MAE and a smaller/sharper representation; the paired test did not detect an Energy Score difference. Nor is the experiment a reproduction of ProbHardE2E's original PDE application. The supported conclusion is a representation-dependent trade-off on a shared affine benchmark.

\end{document}